%% file: main.tex
\pdfoutput=1
\documentclass{article}

\input{headers/header-lqa-head}

\input{headers/header}

\input{headers/header-lqa-tail}

\newcommand{\titlet}{Scalable Certified Segmentation via Randomized Smoothing}

\icmltitlerunning{\titlet}
\begin{document}
\twocolumn[
\icmltitle{\titlet}

\begin{icmlauthorlist}
\icmlauthor{Marc Fischer}{eth}
\icmlauthor{Maximilian Baader}{eth}
\icmlauthor{Martin Vechev}{eth}
\end{icmlauthorlist}

\icmlaffiliation{eth}{Department of Computer Science, ETH Zurich, Switzerland}

\icmlcorrespondingauthor{Marc Fischer}{marc.fischer@inf.ethz.ch}

\icmlkeywords{Machine Learning, ICML, Randomized Smoothing, Adversarial Robustness, Certified Robustness, Segmentation, Sematic Segmentation, Adversarial Examples}

\vskip 0.3in
]

\printAffiliationsAndNotice{}  %

\begin{abstract}
  \input{abstract}
\end{abstract}

\input{intro} %
\input{related}

\input{smoothing}
\input{segmentation}

\input{scaleable_segmentation}

\input{evaluation}
\input{conclusion}

\message{^^JLASTBODYPAGE \thepage^^J}

\clearpage
\bibliography{references}
\bibliographystyle{icml2021}

\message{^^JLASTREFERENCESPAGE \thepage^^J}

\ifincludeappendixx
	\clearpage
	\appendix
	\include{appendix}

\fi

\end{document}

%% file: headers/header-lqa-head.tex
\usepackage[utf8]{inputenc} %

\usepackage{graphicx}

\usepackage[T1]{fontenc}

\input{headers/lqa/overfull}

\input{headers/lqa/comments}
\input{headers/lqa/appendix}

\input{headers/lqa/acronyms}

\input{headers/lqa/listing}

%% file: headers/lqa/overfull.tex
\IfFileExists{headers/config/showoverfull.config}{
	\overfullrule=1cm
}{
}

%% file: headers/lqa/appendix.tex
\usepackage{etoolbox}

\newbool{includeappendix}
\setbool{includeappendix}{true} %
\IfFileExists{headers/config/noappendix.config}{
	\setbool{includeappendix}{false}
}{}

\newif\ifincludeappendixx
\ifbool{includeappendix}{
	\includeappendixxtrue
}{
	\includeappendixxfalse
}

\usepackage{xr} %
\usepackage{filecontents}

\ifbool{includeappendix}{}{
	\input{appendix-labels-loader}

	\externaldocument{appendix-labels}
}

%% file: headers/lqa/acronyms.tex
\usepackage{acro} %

\DeclareAcronym{cli} {
    short = CLI,
    long = Command Line Interface,
    class = abbrev
}

%% file: headers/lqa/listing.tex
\usepackage{listings}

\usepackage{textcomp}

\usepackage{xcolor}

\usepackage[scaled=0.8]{beramono}

\definecolor{ckeyword}{HTML}{7F0055}
\definecolor{ccomment}{HTML}{3F7F5F}
\definecolor{cstring}{HTML}{2A0099}

\lstdefinestyle{numbers}{
	numbers=left,
	framexleftmargin=20pt,
	numberstyle=\tiny,
	firstnumber=auto,
	numbersep=1em,
	xleftmargin=2em
}

\lstdefinestyle{layout}{
	frame=none,
	captionpos=b,
}

\lstdefinestyle{comment-style}{
	morecomment=[l]//,
	morecomment=[s]{/*}{*/},
	commentstyle={\color{ccomment}\itshape},
}

\lstdefinestyle{string-style}{
	morestring=[b]",%
	morestring=[b]',%
	stringstyle={\color{cstring}},
	showstringspaces=false,%
}

\lstdefinestyle{keyword-style}{
	keywordstyle={\ttfamily\bfseries},
	morekeywords={
		function,
		constructor,
		int,
		bool,
		return,
		returns,
		uint
	},
	morekeywords = [2]{},
	keywordstyle = [2]{\text},
	sensitive=true,
}

\lstdefinestyle{input-encoding}{
	inputencoding=utf8,
	extendedchars=true,
	literate=
	{ℝ}{$\reals$}1%
	{→}{$\rightarrow$}1%
	{α}{$\alpha$}1%
	{β}{$\beta$}1%
	{λ}{$\lambda$}1%
	{θ}{$\theta$}1%
	{ϕ}{$\phi$}1%
}

\lstdefinestyle{escaping}{
	moredelim={**[is][\color{blue}]{\%}{\%}},
	escapechar=|,
	mathescape=true
}

\lstdefinestyle{default-style}{
	basicstyle=\fontencoding{T1}\ttfamily\footnotesize,
	style=numbers,
	style=layout,
	style=comment-style,
	style=string-style,
	style=keyword-style,
	style=input-encoding,
	style=escaping,
	tabsize=2,
	upquote=true
}

\lstdefinelanguage{BASIC}{
	language=C++,
	style=default-style
}[keywords,comments,strings]%

%% file: headers/header.tex
  \usepackage{microtype}
  \usepackage{graphicx}
  \usepackage{subcaption}
  \usepackage{booktabs} %
  \usepackage{hyperref}
  
  \usepackage[accepted]{icml2021}

  \usepackage{amsthm}
  \usepackage{amsmath}
\usepackage{mathtools}
\usepackage{bbold}
\usepackage{xspace}
\usepackage{multirow}
\usepackage{ifthen}
\usepackage{nth}

\input{math_commands}

\newcommand{\prob}[0]{\mathbb{P}} %
\theoremstyle{plain} %
\newtheorem{theorem}{Theorem}[section]
\newtheorem*{theorem*}{Theorem}

\newtheorem*{lemma*}{Lemma}

\newtheorem*{corollary*}{Corollary}
\newtheorem{proposition}{Proposition}
\newtheorem{proposition*}{Proposition*}

\theoremstyle{definition}

\newcommand{\abstain}{{\ensuremath{\oslash}}\xspace}
\newcommand{\pval}{{\ensuremath{pv}}\xspace}
\newcommand{\naiveCartesian}{\textsc{JointClass}\xspace}
\newcommand{\naiveMulti}{\textsc{IndivClass}\xspace}

%% file: math_commands.tex
\usepackage{amsmath,amsfonts,bm}

\def\1{\bm{1}}

\def\vc{{\bm{c}}}

\def\vx{{\bm{x}}}
\def\vy{{\bm{y}}}

\def\mR{{\bm{R}}}

\DeclareMathAlphabet{\mathsfit}{\encodingdefault}{\sfdefault}{m}{sl}
\SetMathAlphabet{\mathsfit}{bold}{\encodingdefault}{\sfdefault}{bx}{n}

\DeclareMathOperator*{\argmax}{arg\,max}

%% file: headers/header-lqa-tail.tex
\input{headers/lqa/references}

%% file: headers/lqa/references.tex
\usepackage[capitalize]{cleveref}

\makeatletter
\newcommand\theHALG@line{\thealgorithm.\arabic{ALG@line}}
\makeatother

\crefformat{section}{\S#2#1#3}

\crefrangeformat{section}{\S#3#1#4\crefrangeconjunction\S#5#2#6}

\crefmultiformat{section}{\S#2#1#3}{\crefpairconjunction\S#2#1#3}{\crefmiddleconjunction\S#2#1#3}{\creflastconjunction\S#2#1#3}

\newcommand{\crefrangeconjunction}{--}

\crefname{listing}{Lst.}{listings}
\crefname{line}{Lin.}{Lin.}
\crefname{appendix}{App.}{App.}

\newcommand{\appref}[1]{%
	\ifbool{includeappendix}{\cref{#1}}{the appendix}%
}
\newcommand{\Appref}[1]{%
	\ifbool{includeappendix}{\cref{#1}}{The appendix}%
}

%% file: abstract.tex
We present a new certification method for image and point cloud segmentation based on randomized smoothing. The method leverages a novel scalable algorithm for prediction and certification that correctly accounts for multiple testing, necessary for ensuring statistical guarantees. The key to our approach is reliance on established multiple-testing correction mechanisms as well as the ability to abstain from classifying single pixels or points while still robustly segmenting the overall input. Our experimental evaluation on synthetic data and challenging datasets, such as Pascal Context, Cityscapes, and ShapeNet, shows that our algorithm can achieve, for the first time, competitive accuracy and certification guarantees on real-world segmentation tasks.
We provide an implementation at \url{https://github.com/eth-sri/segmentation-smoothing}.

%% file: intro.tex
\section{Introduction}
\label{sec:intro}

Semantic image segmentation and point cloud part segmentation are important problems in many safety critical domains including medical imaging \cite{perone2018spinal} and autonomous driving \cite{DengYQWW17}. However, deep learning models used for segmentation are vulnerable to adversarial attacks \cite{XieWZZXY17, ArnabMT18, XiangQL19}, preventing their application to such tasks.
This vulnerability is illustrated in \cref{fig:intro}, where the task is to segment the (adversarially attacked) image shown in \cref{fig:intro:original}. We see that the segmentation of the adversarially attacked image is very different from the ground truth depicted in \cref{fig:intro:groundtruth}, potentially causing unfavorable outcomes.
While provable robustness to such adversarial perturbations is well studied for classification
\cite{KatzBDJK17, GehrMDTCV18, WongK18, CohenRK19},
the investigation of certified segmentation just begun recently
\citep{lorenzRBSV2021robustness,tran2021robustness}.

Certifiably robust segmentation is a challenging task, as the classification of each
component (e.g., a pixel in an image) needs to be certified simultaneously. Many datasets and models in
this domain are beyond the reach of current deterministic verification methods while
probabilistic methods need to account for accumulating uncertainty over the large number of individual certifications.

In this work we propose a novel method to certify the robustness of segmentation models via randomized smoothing \citep{CohenRK19}, a probabilistic certification method able to certify $\ell_2$ robustness around large images.
As depicted in
\cref{fig:intro:certified}, our method enables the certification of challenging segmentation tasks by certifying each component individually, abstaining from unstable ones that cause naive algorithms to fail.
This abstention mechanism also provides strong synergies with the multiple testing correction, required for the soundness of our approach, thus enabling high certification rates.

\input{example_fig.tex}

While we focus in our evaluation on $\ell^2$ robustness, our method is general and can also combined with randomized smoothing methods that certify robustness to other $\ell^p$-bounded attacks or parametrized transformations like rotations \citep{FischerBV20}.

\textbf{Main Contributions} Our key contributions are:
\begin{itemize}
  \item We investigate the obstacles of scaling randomized smoothing from
  classification to segmentation, identifying two key challenges: the influence of single bad components and the  multiple testing trade-offs
  (\cref{sec:segmentation})
  \item We introduce a scalable algorithm that addresses these issues, allowing,
        for the first time, to certify large scale segmentation models (\cref{sec:scaleable}).
  \item We show that this algorithm can be applied to different generalizations
        of randomized smoothing, enabling defenses against different attacker
        models (\cref{sec:extensions}).
  \item We provide an extensive evaluation on semantic
        image segmentation and point cloud part segmentation, achieving up to
        88\% and 55\% certified pixel accuracy on Cityscapes and Pascal context
        respectively, while obtaining mIoU of 0.6 and 0.2 (\cref{sec:eval}).
\end{itemize}

%% file: example_fig.tex
\begin{figure}
    \newcommand{\subfigwidth}{0.49\columnwidth}
    \begin{subfigure}[b]{\subfigwidth}
        \includegraphics[width=\textwidth]{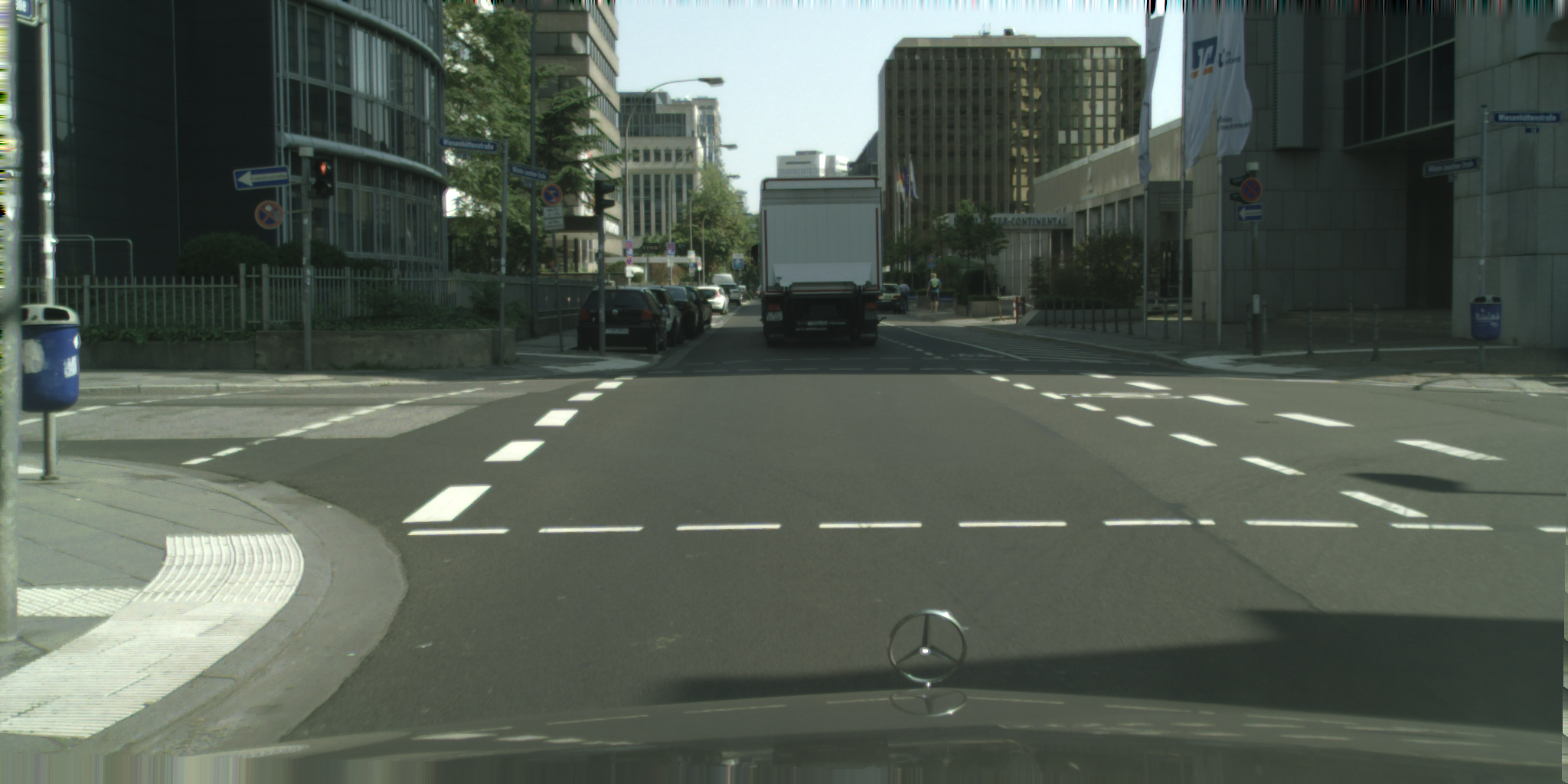}
        \caption{Attacked image} 
        \label{fig:intro:original}
    \end{subfigure}
    \hfill
    \begin{subfigure}[b]{\subfigwidth}
        \includegraphics[width=\textwidth]{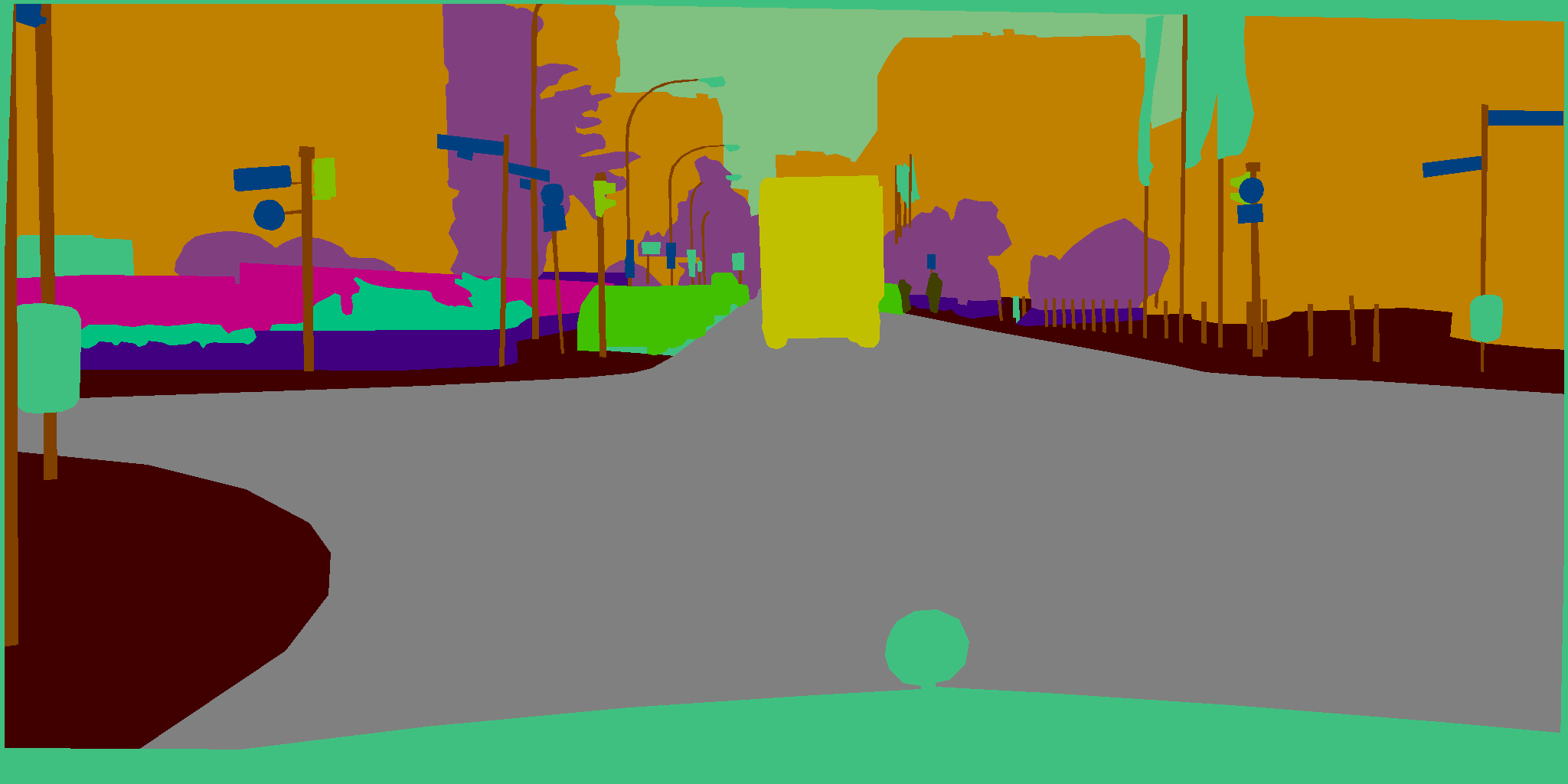}
        \caption{Ground truth segmentation} 
        \label{fig:intro:groundtruth}
    \end{subfigure}
    \linebreak
    \begin{subfigure}[b]{\subfigwidth}
        \includegraphics[width=\textwidth]{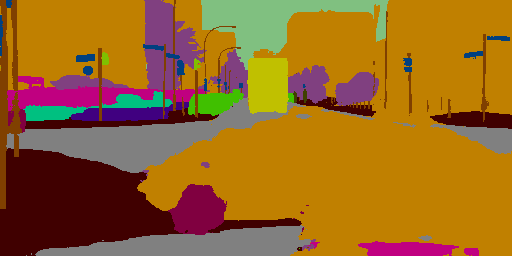}
        \caption{Attacked segmentation} 
        \label{fig:intro:attacked}
    \end{subfigure}
    \hfill
    \begin{subfigure}[b]{\subfigwidth}
        \includegraphics[width=\textwidth]{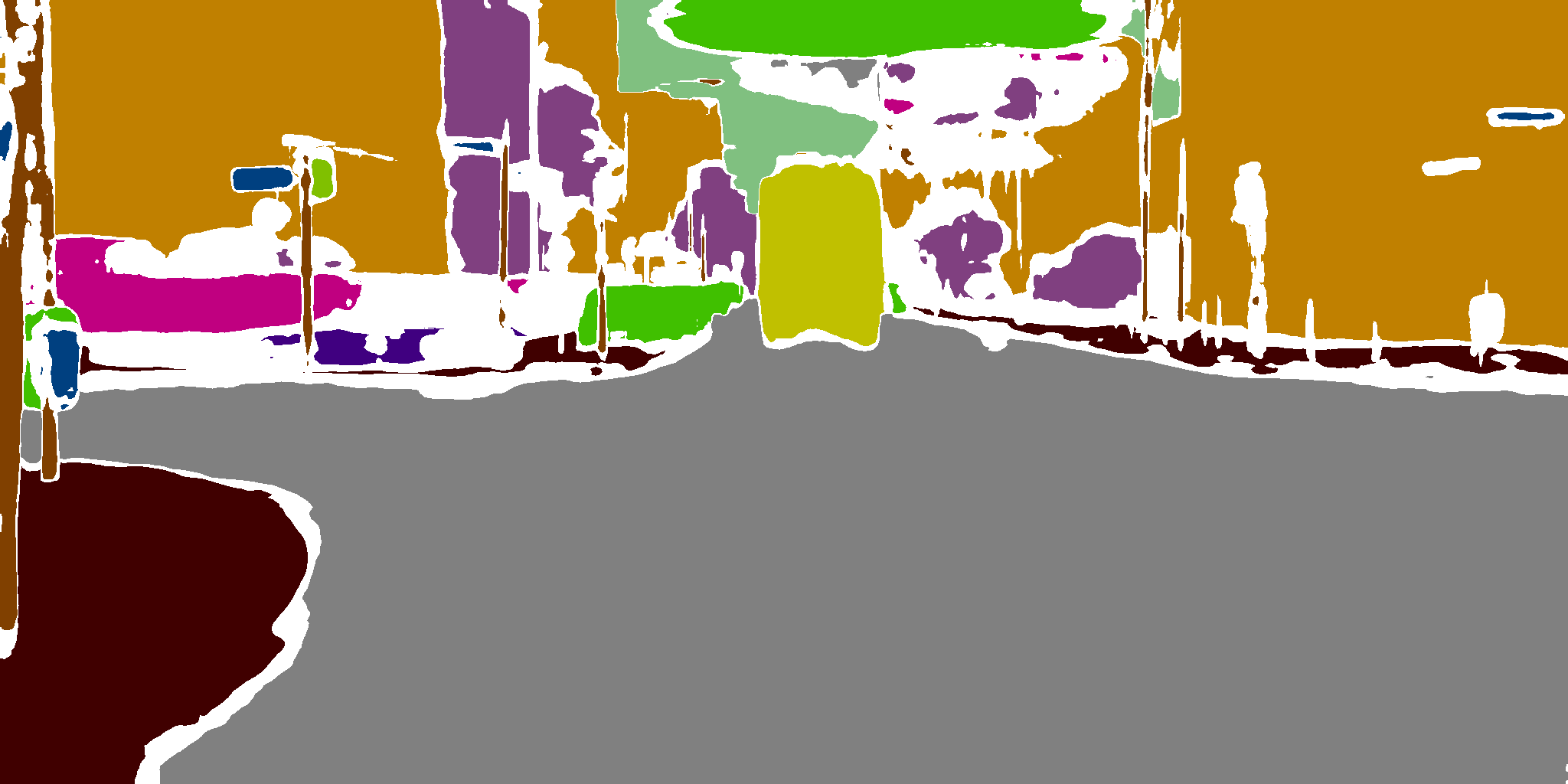}
        \caption{Certified segmentation} 
        \label{fig:intro:certified}
    \end{subfigure}
    \caption{In semantic segmentation, a model segments an input (a) by
      classifying each pixel. While the result should match (b), a non-robust model
      predicts (c) as the input (a) was perturbed by additive $\ell_{2}$ noise (PGD). Our model is certifiably robust to this perturbation (d) by
      abstaining from ambiguous pixels (white). The model abstains where multiple classes meet, causing ambiguity. We provide technical details and further examples in
      \cref{sec:app-details-attacks}. }
    \label{fig:intro}
\end{figure}

%% file: related.tex
\section{Related Work}
\label{sec:related}

In the following, we survey the most relevant related work.

\paragraph{Adversarial Attacks}
\citet{BiggioCMNSLGR13,SzegedyAdvExamples} discovered adversarial examples,
which are inputs perturbed in a way that preserves their semantics but fools deep networks. 
To improve performance on these inputs, training can be extended to including adversarial examples, called adversarial training
\citep{KurakinGB17, MadryMSTV18}.
Most important to this work are attacks on semantic segmentation: \citet{XieWZZXY17} introduced a targeted gradient based unconstrained attack, by maximizing the summed individual targeted losses for every pixel. \citet{ArnabMT18} applied FGSM based attacks \citep{GoodfellowAdvExamples} to the semantic segmentation setting. Similarly, the vulnerability of point cloud classifiers was exposed in \citet{XiangQL19, LiuYS19, SunCCM20}.

\paragraph{Certified robustness and defenses}
However, adversarially trained neural networks do not come with robustness guarantees. To address this, different certification methods have been proposed recently using various methods, relying upon SMT solvers \citep{KatzBDJK17, Ehlers17}, semidefinite programming \citep{RaghunathanSL18a} and linear relaxations \citep{GehrMDTCV18, ZhangWCHD18, WangPWYJ18, WengZCSHDBD18, WongK18, SinghGPV19b, SinghGPV19a}. Specifically, linear relaxations have been used beyond the classical $\ell_p$ noise setting to certify against geometric transformations \citep{SinghGPV19b, BalunovicBSGV19, MohapatraWC0D20} and vector field attacks \citep{ruoss2020spatial}.

To further improve certification rates, methods that train networks to be certifiable have been proposed \citep{RaghunathanSL18b, MirmanGV18, GowalDSBQUAMK18, BalunovicV20}.

Notable to our setting are \citet{lorenzRBSV2021robustness,tran2021robustness}, who extend deterministic certification to point cloud segmentation and semantic segmentation respectively. However, due to the limitations of deterministic certification these models are small in scale.
Further, \citet{BielikV20,sheikholeslami2021provably} improved robust classifiers with the ability to abstain from classification.

\paragraph{Randomized Smoothing}
Despite all this, 
deterministic certification performance on complicated datasets remained unsatisfactory. Recently, based on \citet{LiSampling} and \citet{LecuyerAG0J19}, \citet{CohenRK19} presented randomized smoothing, which was the first certification method to successfully certify $\ell_2$ robustness of large neural networks on large images. \citet{SalmanLRZZBY19} improved the results, by combining the smoothing training procedure with adversarial training. \citet{YangDHSR020} derive conditions for optimal smoothing distributions for $\ell_1, \ell_2$ and $\ell_\infty$ adversaries, if only label information is available. \citet{MohapatraKWC0D20} incorporated gradient information to improve certification radii. \citet{ZhaiDHZGRH020,JeongS20} improved the training procedure for base models by introducing regularized losses.

Randomized smoothing has been extended in various ways. \citet{BojchevskiKG20} proposed a certification scheme suitable for discrete data and applied it successfully to certify graph neural networks.
\citep{LevineF20} certified robustness against Wasserstein adversarial examples.
\citet{FischerBV20} and \citet{li2020tss} used randomized smoothing to certify robustness against geometric perturbations. \citet{SalmanSYKK20} showed that using a denoiser, of the shelf classifiers can be turned into certifiable classifiers without retraining. \citet{0001F20a} and \citet{lin2021certified} presented methods to certify robustness against adversarially placed patches.

Most closely related to this work are \citet{ChiangCAK0G20}, which introduces
median smoothing and applies it to certify object detectors, and
\citet{schuchardt2021collective}, which also extends randomized smoothing to
collective robustness certificates over multiple components. They specifically
exploit the locality of the classifier to the data, making their defense
particularly suitable for graphs where an attacker can only modify certain
subsets. While their approach can in principle be applied to semantic
segmentation, modern approaches commonly rely on global information.

%% file: smoothing.tex
\section{Randomized Smoothing for Classification}
\label{sec:smoothing}

\begin{algorithm}[t]
   \caption{adapted from \citep{CohenRK19}}
   \label{alg:cohen}
\begin{algorithmic}
   \STATE \textit{\# evaluate $\bar{f}$ at $\vx$}
   \STATE \textbf{function} \textsc{Predict}($f$, $\sigma$, $\vx$, $n$, $\alpha$)
   \STATE \quad \texttt{cnts} $\leftarrow$ \textsc{Sample}($f$, $\vx$, $n$, $\sigma$)
   \STATE \quad $\hat{c}_A, \hat{c}_B \leftarrow$ top two indices in \texttt{cnts}
   \STATE \quad $n_A, n_B \leftarrow $ \texttt{cnts}[$\hat{c}_A$],  \texttt{cnts}[$\hat{c}_B$]
   \STATE \quad \textbf{if} \textsc{BinPValue}($n_A$, $n_A + n_B$, $=$, 0.5) $\le \alpha$ \textbf{return} $\hat{c}_A$
   \STATE \quad \textbf{else return} \abstain
   \STATE
   \STATE \textit{\# certify the robustness of $\bar{f}$ around $\vx$}
   \STATE \textbf{function} \textsc{Certify}($f$, $\sigma$, $\vx$, $n_0$, $n$, $\alpha$)
   \STATE \quad $\texttt{cnts}^{0} \leftarrow \textsc{Sample}(f, \vx, n_0, \sigma)$
   \STATE \quad $\hat{c}_A \leftarrow$ top index in $\texttt{cnts}^{0}$
   \STATE \quad $\texttt{cnts} \leftarrow \textsc{Sample}(f, \vx, n, \sigma)$
   \STATE \quad $\underline{p_A} \leftarrow \textsc{LowerConfBnd}$($\texttt{cnts}[\hat{c}_A]$, $n$, $1 - \alpha$)
   \STATE \quad \textbf{if} $\underline{p_A} > \frac{1}{2}$ \textbf{return} prediction $\hat{c}_A$ and radius $\sigma \, \Phi^{-1}(\underline{p_A})$
   \STATE \quad \textbf{else return} \abstain
\end{algorithmic}
\end{algorithm}

In this section we will briefly review the necessary background and notation on
randomized smoothing, before extending it in \cref{sec:segmentation,sec:scaleable}.
Randomized smoothing \citep{CohenRK19} constructs a robust (smoothed)
classifier $\bar{f}$ from a (base) classifier $f$. The classifier $\bar{f}$ is
then provably robust to $\ell_{2}$-perturbations up to a certain radius.

Concretely, for a classifier $f \colon \mathbb{R}^m \to \mathcal{Y}$ and random
variable $\epsilon \sim \mathcal{N}(0, \sigma^2 \mathbb{1})$: we define a smoothed classifier $\bar{f}$ as
\begin{equation}
  \label{eq:g}
    \bar{f}(\vx) := \argmax_c \prob_{\epsilon \sim \mathcal{N}
        (0, \sigma^2 \mathbb{1})}(f(\vx + \epsilon)=c).
\end{equation}

This classifier $\bar{f}$ is then robust to adversarial perturbations:

\begin{theorem}[From \cite{CohenRK19}] \label{thm:original}
    Suppose $c_A \in \mathcal{Y}$, $\underline{p_A}, \overline{p_B} \in [0,1]$. If
    \begin{equation*}
        \prob_{\epsilon}(f(\vx + \epsilon)=c_A)
        \geq
        \underline{p_A}
        \geq
        \overline{p_B}
        \geq
        \max_{c \neq c_A}\prob_{\epsilon}(f(\vx + \epsilon)=c),
    \end{equation*}
    then $\bar{f}(\vx + \delta) = c_A$ for all $\delta$ satisfying $\|\delta\|_2 \leq R$
    with $R := \tfrac{\sigma}{2}(\Phi^{-1}(\underline{p_A}) - \Phi^{-1}(\overline{p_B}))$.
\end{theorem}

In order to evaluate the model $\bar{f}$ and calculate its robustness radius
$R$, we need to be able to compute $\underline{p_{A}}$ for the input $\vx$.
However, since the computation of the probability in \cref{eq:g} is not
tractable (for most choices of $f$) $\bar{f}$ can not be evaluated exactly.
To still query it, \citet{CohenRK19} suggest \textsc{Predict} and
\textsc{Certify} (\cref{alg:cohen}) which use Monte-Carlo sampling to
approximate $\bar{f}$. Both algorithms utilize the procedure
\textsc{Sample}, which samples $n$ random realizations of
$\epsilon \sim \mathcal{N}(0, \sigma)$ and computes $f(\vx + \epsilon)$, which is returned as a vector
of counts for each class in $\mathcal{Y}$. These samples are then used to
estimate the class $c_{A}$ and radius $R$ with confidence $1 - \alpha$, where $\alpha \in [0, 1]$.
\textsc{Predict} utilizes a two-sided binomial p-value test to determine
$\bar{f}(\vx)$.
With probability at most $\alpha$ it will abstain, denoted as \abstain (a predefined value), else
it will produce $\bar{f}(\vx)$.
\textsc{Certify} uses the Clopper-Pearson confidence interval
\citep{clopper1934use} to soundly estimate $\underline{p_{A}}$ and then invoke
\cref{thm:original} with $\overline{p_{B}} = 1 - \underline{p_{A}}$ to obtain
$R$. If \textsc{Certify} returns a class other than \abstain and a radius $R$,
then with probability $1-\alpha$ the guarantee in \cref{thm:original} holds for
this $R$. The certification radius $R$ increases if (i) the value of
$\underline{p_{A}}$ increases, which increases if the classifier $f$ is robust
to noise, (ii) the number of samples $n$ used to estimate it increases, or (iii)
the $\alpha$ increases which means that the confidence decreases.

%% file: segmentation.tex
\section{Randomized Smoothing for Segmentation}
\label{sec:segmentation}

To show how randomized smoothing can be applied in the segmentation setting we
first discuss the mathematical problem formulation and two direct adaptations
of randomized smoothing to the problem. By outlining how these fail in practice,
we determine the two key challenges preventing their success. Then in,
\cref{sec:scaleable}, we address these challenges.

\paragraph{Segmentation}
Given an input $\vx = \{\vx_{i}\}_{i=0}^{N}$ of $N$ components $\vx_{i} \in
\mathcal{X}$ (e.g., points or pixels) and a set of possible classes $\mathcal{Y}$,
segmentation can be seen as a function $f \colon \mathcal{X}^{N} \to \mathcal{Y}^{N}$.
That is, to each $\vx_{i}$ we assign a $f_{i}(\vx) = y_{i} \in \mathcal{Y}$,
where $f_{i}$ denotes the $i$-th component of the output of $f$ invoked on input
$\vx$. Here we assume $\mathcal{X} := \mathbb{R}^{m}$. Unless specified, we will
use $m=3$ as this allows for RGB color pixels as well as 3d point clouds.

\paragraph{Direct Approaches}
To apply randomized smoothing, as introduced in \cref{sec:smoothing}, to
segmentation, we can reduce it to one or multiple classification problems.

We can recast segmentation $f \colon \mathcal{X}^{N} \to \mathcal{Y}^{N}$ as a
classification problem by considering the cartesian product of the co-domain
$\mathcal{V} := \bigtimes_{i=1}^{N} \mathcal{Y}$ and a new function
$f'\colon \mathcal{X}^{N} \to \mathcal{V}$ that performs classification.
Thus we can apply \textsc{Certify} (\cref{alg:cohen}) to the base classifier
$f'$.
This provides a mathematically sound mapping of segmentation to
classification. However, a change in the classification of a single component $\vx_{i}$ will change
the overall class in $\mathcal{V}$ making it hard to find a majority class
$\hat{c}_{A}$ with high $\underline{p_{A}}$ in practice.
We refer to this method as \naiveCartesian, short for ``joint classification''.

Alternatively, rather than considering all components at the same time, we can
also classify each component individually. To this end, we let $f_{i}(\vx)$
denote the $i$-th component of $f(\vx)$ and apply \textsc{Certify}, in
\cref{alg:cohen}, $N$ times to evaluate $\tilde{f_{i}}(\vx)$ to obtain classes
$\hat{c}_{A, 1}, \dots, \hat{c}_{A, N}$ and radii $R_{1}, \dots, R_{N}$.
Then the overall radius is given as $R = \min_{i} R_{i}$ and a single abstention
will cause an overall abstention. To reduce evaluation cost we can reuse the
same input samples for all components of the output vector, that is sample
$f(\vx)$ rather than individual $f_{i}(\vx)$.
We refer to this method as \naiveMulti.
Further, the result of each call to \textsc{Certify} only holds with probability
$1-\alpha$. Thus, by the union bound, the overall probability is limited by $\prob(\bigvee_{i} i\text{-th
  test incorrect}) \leq \min(\sum_{i} \prob(i\text{-th test
  incorrect}), 1) = \min(N \alpha, 1) $, which for large $N$
quickly becomes problematic. This can be compensated by carrying out calls to \textsc{Certify} with
$\alpha' = \frac{\alpha}{N}$, which becomes prohibitively expensive as $n$ needs
to be increased for the procedure to not abstain.

\paragraph{Key Challenges}
These direct applications of randomized smoothing to segmentation suffer from
multiple problems that can be reduced to two challenges:
\begin{itemize}
  \item \emph{Bad Components}: Both algorithms can be forced to abstain or
        report a small radius by a single bad component $\vx_i$ for which the base
        classifier is unstable.
  \item \emph{Multiple Testing Trade-off}: Any algorithm that, like \naiveMulti, reduces
        the certification segmentation to multiple stochastic tests (such as
        \textsc{Certify}) suffers from the multiple testing problem. As outlined
        before, if each partial result only holds with probability $\alpha$ then the
        overall probability decays, using the union bound, linearly in the number of tests. Thus, to
        remain sound one is forced to choose between scalability or low confidence.
\end{itemize}

%% file: scaleable_segmentation.tex
\section{Scalable Certified Segmentation} \label{sec:scaleable} %
We now introduce our algorithm for certified segmentation by
addressing these challenges. In particular, we will side-step the bad component issue and allow for a more favorable trade-off in the multiple-testing setting.

To limit the impact of bad components on the overall result, we introduce a
threshold $\tau \in [\frac{1}{2}, 1)$ and define a model
$\bar{f}^{\tau}: \mathcal{X}^{N} \to \hat{\mathcal{Y}}^{N}$, with
$\hat{\mathcal{Y}} = \mathcal{Y} \cup \{\abstain\}$, that abstains  if the
probability of the top class for component $\vx_{i}$ is below $\tau$:
\begin{align*}
  &\bar{f}^{\tau}_{i}(\vx) =
  \begin{cases}
    c_{A, i} & \text{if } \prob_{\epsilon \sim \mathcal{N}(0, \sigma)}(f_{i}(\vx + \epsilon)) > \tau\\
    \abstain & \text{else}
  \end{cases},
\end{align*}
where $c_{A, i} = \argmax_{c \in \mathcal{Y}}
\prob_{\epsilon \sim \mathcal{N}(0, \sigma)}(f_{i}(\vx + \epsilon) = c)$.

This means that on components with fluctuating classes, the model $\bar{f}^{\tau}$ does not need to commit to a class.
For the model $\bar{f}^{\tau}$, we obtain a safety guarantee similar to the original theorem by \citep{CohenRK19}:

\begin{theorem}
  \label{thm:fbartau}
  Let $\mathcal{I}_{\vx} = \{i \mid \bar{f}_{i}^{\tau}(\vx) \neq \abstain, i \in 1, \dots, N \}$ denote the set of non-abstain indices for $\bar{f}^{\tau}(\vx)$. Then,
  \begin{align*}
    &\bar{f}^{\tau}_{i}(\vx + \delta)
    = \bar{f}^{\tau}_{i}(\vx), \quad \forall i \in \mathcal{I}_{\vx}
  \end{align*}
  for $\delta \in \mathbb{R}^{N \times m} $ with $\|\delta\|_{2} \leq R := \sigma \Phi^{-1}(\tau)$.
\end{theorem}

\begin{proof}
  We consider $\bar{f}_{1}^{\tau}, \dots, \bar{f}_{N}^{\tau}$ independently.
  With $p_{A, i} := \prob_{\epsilon \sim \mathcal{N}(0, \sigma)}(f_{i}(\vx + \epsilon) = c_{A, i}) > \tau$, we
  invoke \cref{thm:original} with $\underline{p_{A}} = \tau$ for $\bar{f}^{\tau}$ and
  obtain robustness radius $R := \sigma \Phi^{-1}(\tau)$ for $\bar{f}^{\tau}_{i}(\vx)$.
  This holds for all $i$ where the class probability $p_{A, i} > \tau$, denoted by the
  set $\mathcal{I}_{\vx}$.
\end{proof}

\begin{algorithm}[t]
   \caption{algorithm for certification and prediction}
   \label{alg:segcertify}
\begin{algorithmic}
   \STATE \textit{\# evaluate $\bar{f}^{\tau}$ at $\vx$}
   \STATE \textbf{function} \textsc{SegCertify}($f$, $\sigma$, $\vx$, $n$, $n_{0}$, $\tau$, $\alpha$)
   \STATE \quad $\text{\texttt{cnts}}^{0}_{1}, \dots, \text{\texttt{cnts}}^{0}_{N}$ $\leftarrow$ \textsc{Sample}($f$, $\vx$, $n_{0}$, $\sigma$)
   \STATE \quad $\text{\texttt{cnts}}_{1}, \dots,  \text{\texttt{cnts}}_{N}$ $\leftarrow$ \textsc{Sample}($f$, $\vx$, $n$, $\sigma$)
   \STATE \quad\textbf{for} $i \leftarrow \{1, \dots, N\}$:
   \STATE \quad\quad $\hat{c}_{i} \leftarrow$ top index in $\text{\texttt{cnts}}^{0}_{i}$
   \STATE \quad\quad $n_{i} \leftarrow \text{\texttt{cnts}}_{i}[\hat{c}_{i}]$
   \STATE \quad\quad $\pval_{i} \leftarrow $\textsc{BinPValue}($n_i$, $n$, $\leq$, $\tau$)
   \STATE \quad $r_{1}, \dots, r_{N} \leftarrow $ \textsc{FwerControl}($\alpha$, $\pval_{1}$, $\dots$, $\pval_{N}$)
   \STATE \quad\textbf{for} $i \leftarrow \{1, \dots, N\}$:
   \STATE \quad\quad \textbf{if} $\lnot r_{i}$: $\hat{c}_{i} \leftarrow \abstain$
   \STATE \quad $R \leftarrow \sigma \Phi^{-1}(\tau)$
   \STATE \quad \textbf{return} $\hat{c}_{1}, \dots, \hat{c}_{N}$, $R$
\end{algorithmic}
\end{algorithm}

\paragraph{Certification}
Similarly to \citet{CohenRK19}, we cannot invoke our theoretically constructed
model $\bar{f}^{\tau}$ and must approximate it. The simplest way to do so would be
to invoke \textsc{Certify} for each component and replace the check
$\underline{p_{A}} > \frac{1}{2}$ by $\underline{p_{A}} > \tau$. However, while
this accounts for the bad component issue, it still suffers from the
outlined multiple testing problem.
To address this issue, we now introduce
the \textsc{SegCertify} procedure in \cref{alg:segcertify}.

We will now briefly outline the steps in \cref{alg:segcertify} before describing
them in further detail and arguing about its correctness and properties.
\cref{alg:segcertify} relies on the same primitives as  \cref{alg:cohen},  as well as the \textsc{FwerControl} function, which performs multiple-testing correction, which we will formally introduce shortly.
As before, \textsc{Sample} denotes the evaluation of samples $f(\vx + \epsilon)$
and $\texttt{cnts}_{i}$ denotes the class frequencies observed for the $i$-th
component. Similar to \textsc{Certify}, we use two sets of samples \texttt{cnts}
and $\texttt{cnts}^{0}$ to avoid a model selection bias (and thus invalidate our
statistical test): We use $n_{0}$ samples, denoted $\texttt{cnts}^{0}_{i}$ to
guess the majority class $c_{A, i}$ for the $i$-th component, denoted
$\hat{c}_{i}$ and then determine its appearance count $n_{i}$ out of $n$ trails
from $\texttt{cnts}_{i}$. With these counts, we then perform a one-sided
binomial test, discussed shortly, to obtain its
$p$-value $\pval_{i}$. Given these $p$-values we can employ
\textsc{FwerControl},  which determines which tests we
can reject in order to obtain overall confidence $1-\alpha$. The rejection of the
$i$-th test is denoted by the boolean variable $r_{i}$. If we reject the $i$-th
test $(r_{i} = 1)$, we assume its alternate hypothesis
$p_{A,i} > \tau$ and return the class $c_{A, i}$, else we return \abstain.

To establish the correctness of the approach and improve on the multiple testing
trade-off, we will now briefly review statistical (multiple) hypotheses
testing.

\paragraph{Hypothesis Testing}
Here we consider a fixed but arbitrary single component $i$: If we guessed the majority class $c_{A, i}$ for the $i$-th component to be
$\hat{c}_{i}$, we assume that $f_{i}(\vx + \epsilon)$ returns $\hat{c}_{i}$ with
probability $p_{\hat{c}_{i}, i}$ (denoted $p$ in the following). We now want to
infer wether $p > \tau$ or not, to determine whether $\bar{f}^{\tau}$ will output
$\hat{c}_{i}$ or \abstain. We phrase this as a statistical test with null
hypothesis $(H_{0, i})$ and alternative $(H_{A, i})$:
\begin{align*}
 (H_{0, i}): \;p \leq \tau \qquad\qquad
 (H_{A, i}): \; p > \tau
\end{align*}

A statistical test lets us assume a null hypothesis $(H_{0, i})$ and check how
plausible it is given our observational data.
We can calculate the $p$-value $\pval$, which denotes the probability of the observed
data or an even more extreme event under the null hypothesis. Thus, in our case
rejecting $(H_{0, i})$ means returning $\hat{c}_{i}$ while accepting it means returning \abstain.

Rejecting the null hypothesis when it is actually true (in our case returning a
class if we should abstain) is called a type I error. The opposite, not
rejecting the null hypothesis when it is false is called a type II error.
In our setting type II errors mean additional abstentions on top of those $\bar{f}^{\tau}$ makes by design.
Making a sound statement about robustness means controlling the
type I error while reducing type II errors means fewer abstentions due to the testing procedure.

Commonly, a test is rejected if its $p$-value is below some threshold
$\alpha$, as in the penultimate line of \textsc{Predict} (\cref{alg:cohen}). This bounds the
probability of type I error at probability (or ``level'') $\alpha$.

\paragraph{Multiple Hypothesis Testing}
Multiple tests performed at once require additional care. Usually, if we
perform $N$ tests, we want to bound the probability of any type I error
occurring. This probability is called the \emph{family wise error rate} (FWER). The goal of FWER control is, given a set of $N$ tests and $\alpha$, to reject tests such that the FWER is limited at $\alpha$.
In \cref{alg:segcertify} this is denoted as \textsc{FwerControl}, which is a
procedure that takes the value $\alpha$ and $p$-values
$\pval_{1}, \dots, \pval_{N}$ and decides on rejections $r_{1}, \dots, r_{N}$.

The simplest procedure for this is the Bonferroni method
\citep{bonferroni1936teoria}, which rejects individual tests with $p$-value
$\pval_{i} \leq \frac{\alpha}{N}$. It is based on the same consideration we applied for the
failure probability of \naiveMulti in \cref{sec:segmentation}.
While this controls the FWER at level $\alpha$ it also increases the type II error,
and thus would reduce the number of correctly classified components. To maintain
the same number of correctly classified components we would need to increase $n$
or decrease $\alpha$.

A better alternative to the Bonferroni method is the Holm correction (or
Holm-Bonferroni method) \citep{holm1979simple} which orders the tests by
acceding $p$-value and steps through them at levels
$\frac{\alpha}{N}, \dots, \frac{\alpha}{1}$ (rejecting null if the associated p value is smaller then the level) until the first level where no additional test
can be rejected. This method also controls the FWER at level $\alpha$ but allows for
a lower rate of type II errors.

Further improving upon these methods, usually requires additional information.
This can either be specialization to a certain kind of test
\citep{tukey1949comparing}, additional knowledge (e.g., no negative
dependencies between tests, for procedures such as \citep{vsidak1967rectangular}) or an
estimate of the dependence structure computed via Bootsrap or permutation
methods \citep{westfall1993resampling}.

While our formulation of \textsc{SegCertify} admits all of these corrections, Bootstrap or permutaion procedures are generally infeasible for large $N$.
Thus, in the rest of the paper we consider both Holm and Bonferroni correction and compare them in \cref{sec:synthetic}.

We note that while methods \citep{westfall1985simultaneous} exist for assessing
that $N$ Clopper-Pearson confidence intervals hold jointly at level $\alpha$, thus
allowing better correction for \naiveMulti, these methods do not scale to
realistic segmentation problems as they require re-sampling operations that are
expensive on large problems.

\subsection{Properties of \textsc{SegCertify}}
\label{sec:properties}

\textsc{SegCertify} is a conservative algorithm, as it will rather abstain from
classifying a component than returning a wrong or non-robust result. We
formalize this as well as a statement about the correctness of
\textsc{SegCertify} in \cref{prob:correctness}.

\begin{proposition}
  \label{prob:correctness}
  Let $\hat{c}_{1}, \dots, \hat{c}_{N}$ be the output of \textsc{SegCertify} for input $\vx$
  and $\hat{\mathcal{I}}_{\vx} := \{i \mid \hat{c}_{i} \neq \abstain \}$. Then with probability at least $1-\alpha$ over the randomness in \textsc{SegCertify} $\hat{\mathcal{I}}_{\vx} \subseteq \mathcal{I}_{\vx}$, where $\mathcal{I}_{\vx}$ denotes the previously defined non-abstain indices of $\bar{f}^{\tau}(\vx)$. Therefore $\hat{\vc}_{i} = \bar{f}_{i}^{\tau}(\vx) = \bar{f}_{i}^{\tau}(\vx + \delta)$ for $i \in \hat{\mathcal{I}}_{\vx}$ and $\delta \in \mathbb{R}^{N \times m}$ with $\|\delta\|_{2} \leq R$.
\end{proposition}

\begin{proof}
  Suppose that $i \in \hat{\mathcal{I}}_{\vx} \setminus \mathcal{I}_{\vx}$.
  This $i$ then represents a type I error. However, the overall probability of any such type I error is controlled at level $\alpha$ by the \textsc{FwerControl} step.
  Thus with probability at least $1 - \alpha$, $\hat{\mathcal{I}}_{\vx} \subseteq \mathcal{I}_{\vx}$. The rest follows from \cref{thm:fbartau}.
\end{proof}

We note that there are fundamentally two reasons for
\textsc{SegCertify} to abstain from the
classification of a component: either because the true class probability
$p_{A, i}$ is less than or equal to $\tau$ in which case $\bar{f}^{\tau}_{i}$ abstains
by definition or because of a type II error in \textsc{SegCertify}. This type II error can occur either if we guess the wrong $\hat{c}_{i}$ from our $n_{0}$ samples or because we could not gather sufficient evidence to reject $(H_{0,i})$ under $\alpha$-FWER.
\cref{prob:correctness} only makes a statement about type I errors, but not type
II errors, e.g., $\hat{\mathcal{I}}$ could always be the empty set. In general,
the control of type II errors (called ``power'') of a testing procedure can be
hard to quantify without a precise understanding of the underlying procedure (in our case the model $f$).
However, in \cref{sec:synthetic} we will investigate this experimentally.

In contrast to \citet{CohenRK19} we do not provide separate procedures for
prediction and certification, as for a fixed $\tau$ both only need to determine
$p_{A, i} > \tau$ for all $i$ with high probability (w.h.p.).
For faster prediction, the algorithm can be invoked
with larger $\alpha$ and smaller $n$.

For $N=1$, the algorithm exhibits the same time complexity as those proposed by
\citet{CohenRK19}. For larger $N$, time complexity stays the same (plus the time for sorting the p-values in Holm procedure), however due to
type II errors the number of abstentions increases. Counteracting this leads
to the previously discussed multiple-testing trade-off.

\paragraph{Summary}
Having presented \textsc{SegCertify} we now summarize how it addresses the two key challenges. First, the bad component issue is offset by the introduction of a threshold parameter $\tau$. This allows to exclude single components, which would not be provable, from the overall certificate.
Second, due to this threshold we now aim to show a fixed lower bound $\underline{p_{A}} = \tau$, where we can use a binomial p-value test rather than the Clopper-Pearson interval for a fixed confidence. The key difference here is that the binomial test produces a p-value that is small unless the observed rate $\tfrac{n_{i}}{n}$ is close to $\tau$, while the Clopper-Pearson interval aims to show the maximal $\underline{p_{A}}$ for a fixed confidence. This former setting generally lends itself better to the multiple testing setting. In particular for correction algorithms such as Holm this lets small p-values offset larger ones.

\subsection{Generality}
\label{sec:extensions}

While presented in the $\ell_{2}$-robustness setting as in \citet{CohenRK19},
\textsc{SegCertify} is not tied to this and is compatible with many
extensions and variations of \cref{thm:original} in a plug-and-play manner. For example,
\citet{LiSampling,LecuyerAG0J19,YangDHSR020} showed robustness to $\ell_{1}$
(and other $\ell_{p}$) perturbations. \citet{FischerBV20} introduced an extension
that allows computing a robustness radius over the parameter of a parametric transformation.
In both of these, it suffices for practical purposes to update
\textsc{Sample} to produce perturbations from a suitable distribution (e.g.,
rotation or Laplace noise) as well as update the calculation of the radius $R$
to reflect the different underlying theorems. Similarly, in $\bar{f}^{\tau}$ and
\cref{thm:fbartau}, it suffices to update the class probability according to
the sample distribution and radius prescribed by the relevant theorem.

\input{figures/plot_abstain}

\paragraph{Extensions to $k$-FWER}
When we discussed FWER control so far, we bounded the probability of making
any type I error by $\alpha$. However, in the segmentation setting we have many
individual components, and depending on our objective, we may allow a small
budget of few errors. In this setting, we can employ $k$-FWER control
\citep{lehmann2012generalizations}. It controls the probability
of observing $k$ or more type I errors at probability $\alpha$. Thus with probability
$1-\alpha$ we have observed at most $k-1$ type I errors (false non-abstentions).
\citet{lehmann2012generalizations} introduces a procedure similar to Holm
method that allows for $k$-FWER that is optimal (without the use of further
information on test dependence) and can be simply plugged into
\textsc{SegCertify}. We provide further explanation and evaluation in
\cref{sec:app-kfwer}.

%% file: figures/plot_abstain.tex
\begin{figure*}[ht!]
  \centering
  \newcommand{\subfigwidth}{0.3\textwidth}
    \begin{subfigure}[b]{\subfigwidth}
      \includegraphics[width=\textwidth]{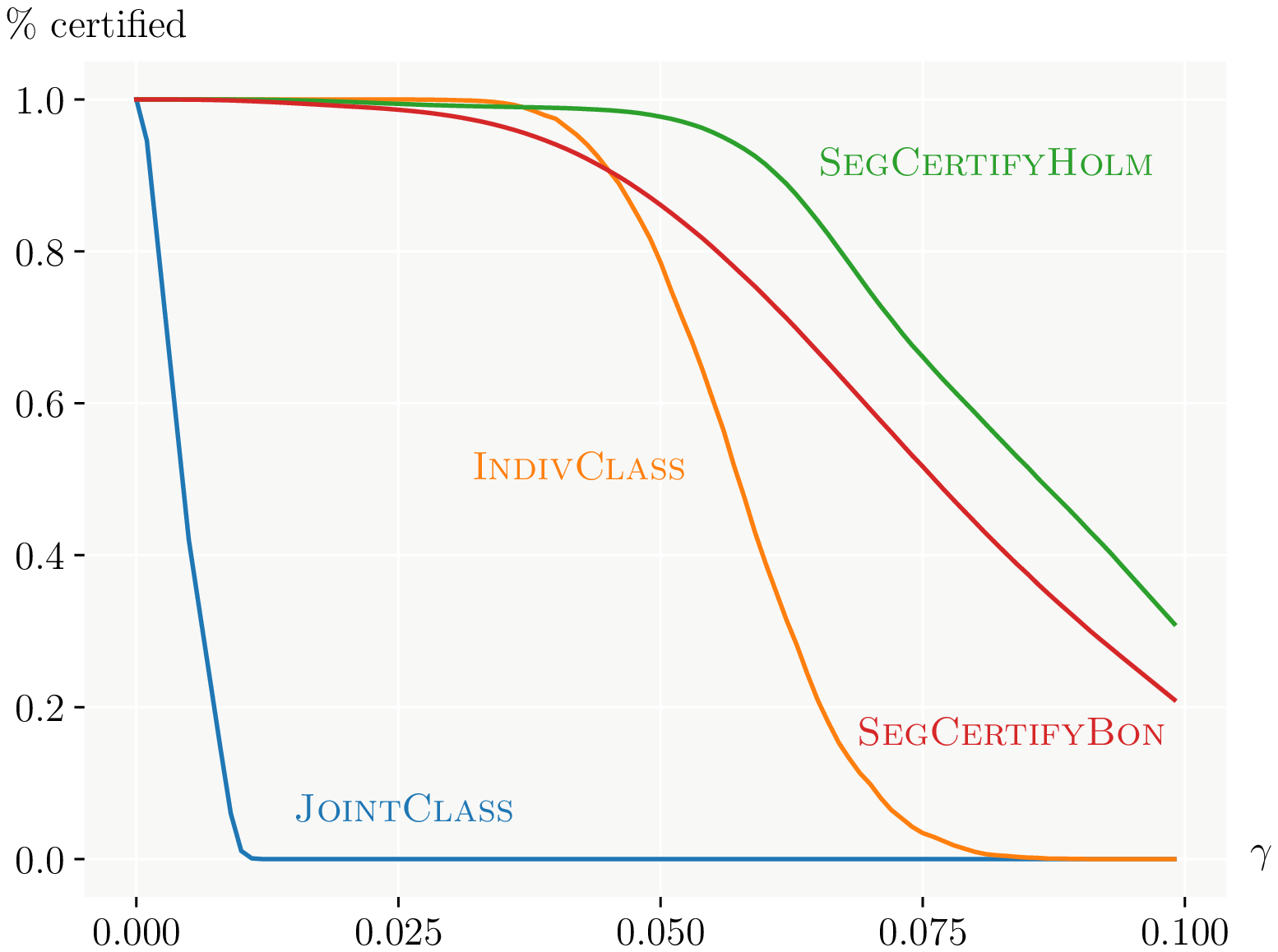}
      \caption{On $N=100$ components, with a classifier that has error rate $5\gamma$ on one component and $\gamma$ on all others.}
        \label{fig:plot_abstain_gamma}
    \end{subfigure}
    \hfill
    \begin{subfigure}[b]{\subfigwidth}
      \includegraphics[width=\textwidth]{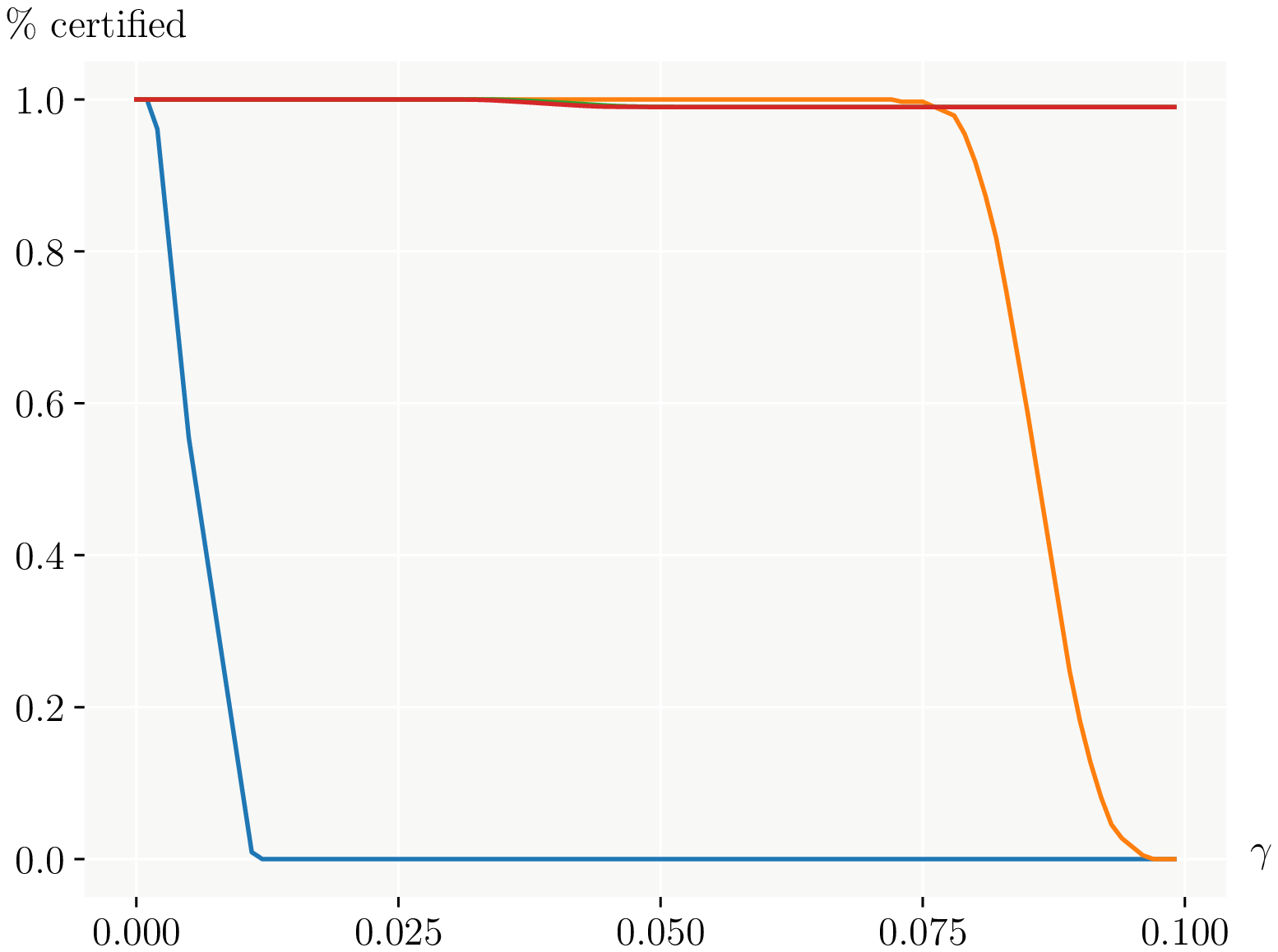}
      \caption{Same setting as \cref{fig:plot_abstain_gamma}, but the algorithms use
        more samples $n_{0}, n$ allowing for less abstentions.}
        \label{fig:plot_abstain_n_gamma}
    \end{subfigure}
    \hfill
    \begin{subfigure}[b]{\subfigwidth}
      \includegraphics[width=\textwidth]{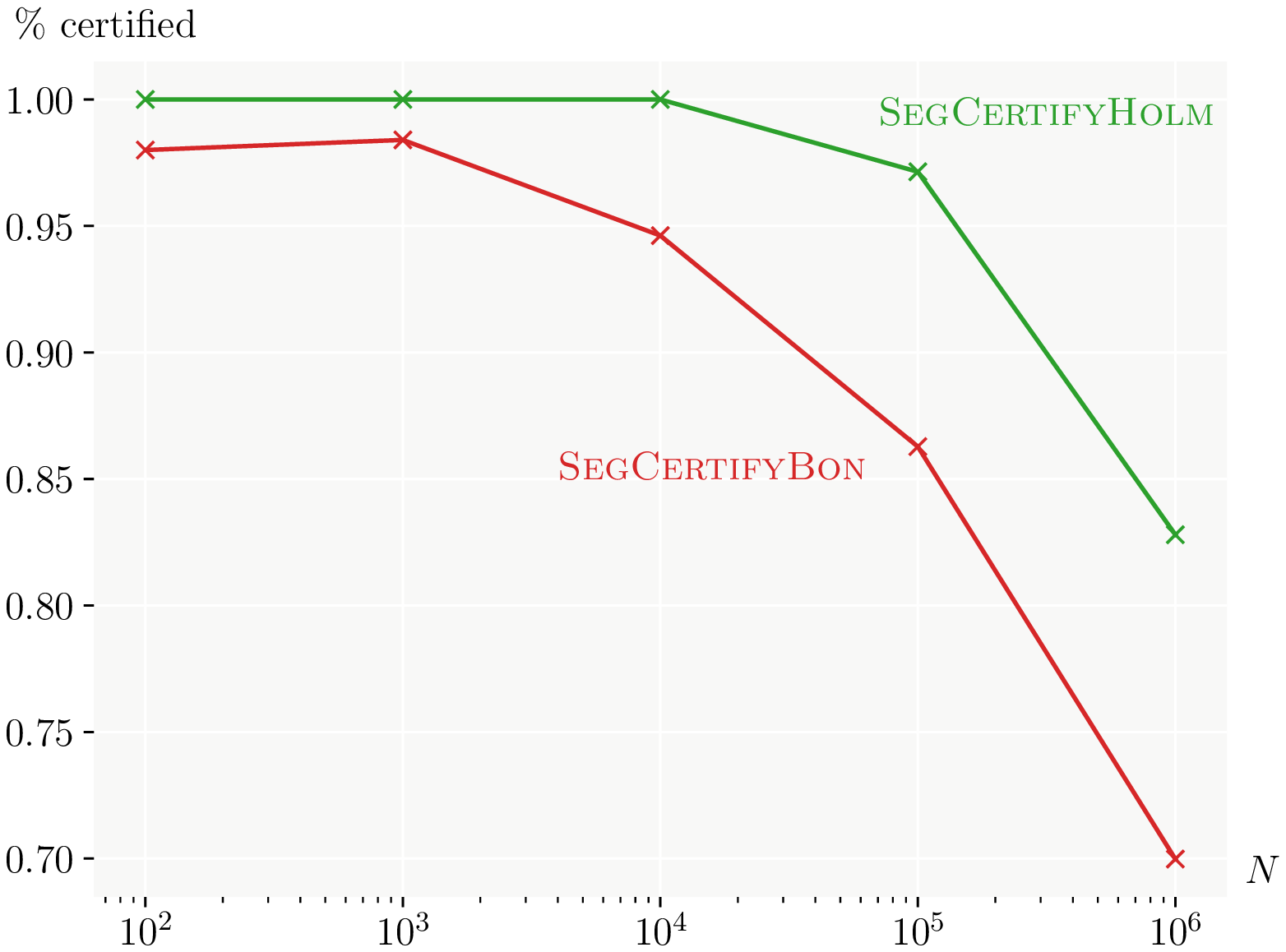}
      \caption{\textsc{SegCertify} with different testing corrections for various numbers of components $N$ with error rate $\gamma = 0.05$.}
        \label{fig:plot_abstain_N}
    \end{subfigure}
    \caption{We empirically investigate the power (ability to avoid type II
      errors -- false abstention) of multiple algorithms on synthetic data. The
      $y$-axis shows the rate of certified (rather than abstained) components. An optimal algorithm would achieve
      $1.0$ or $0.99$ in all plots. }
  \label{fig:plot_abstain}
\end{figure*}

%% file: evaluation.tex
\section{Experimental Evaluation}
\label{sec:eval}

We evaluate our approach in three different settings: (i) investigating the two key challenges on toy data in \cref{sec:synthetic}, (ii)
semantic image segmentation in \cref{sec:eval:semantic}, and (iii) point clouds
part segmentation in \cref{sec:pointcl-part-segm}. All timings are for a single
Nvidia GeForce RTX 2080 Ti. For further details, see \cref{sec:app-experiments}.

\subsection{Toy Data} \label{sec:synthetic}

We now investigate how well \textsc{SegCertify} and the two naive algorithms
handle the identified challenges.

Here we consider a simple setting, where our input is of length $N$
and each component can be from two possible classes. Further, our base model is
an oracle (that does not need to look at the data) and is correct for each of $N-k$
component with probability $1-\gamma$ for some $\gamma \in [0, 1]$ and $k \in {0, \dots, N}$. On the remaining $k$ components it is correct with probability $1-5\gamma$ (clamped to $[0,1]$).

We evaluate \naiveCartesian, \naiveMulti with Bonferroni correction and
\textsc{SegCertify} with Holm, denoted \textsc{SegCertifyHolm}, as well as
Bonferroni correction, denoted \textsc{SegCertifyBon}.
Note that \textsc{SegCertifyBon} in contrast to \naiveMulti employs thresholding.
We investigate the performance of all methods under varying levels of noise by
observing the rate of certified components. As the naive algorithms either
provide a certificate for all pixels or none, we have repeated the experiment
600 times, showing a smoothed average in \cref{fig:plot_abstain_gamma} (the
unsmoothed version can be found in \cref{sec:app-experiments}).
In this setting, we assume $N=100$ components, $k=1$ and $\gamma$ varies between $0$
and $0.1$ along the $x$-axis. All algorithms use $n_{0} = n = 100$ and
$\alpha = 0.001$. \textsc{SegCertify} uses $\tau = 0.75$. This means that for $\gamma \leq 0.05$,
all abstentions are due to the testing procedure (type II errors) and not
prescribed by $\tau$ in the definition of $\bar{f}^{\tau}$. For $0.05 < \gamma \leq 0.1$, there
is one prescribed abstention as $k=1$. Thus the results provide an empirical assessment of the statistical power of the algorithm.

As outlined \cref{sec:segmentation}, \naiveCartesian and \naiveMulti are very
susceptible to a single bad component. While \naiveCartesian almost fails
immediately for $\gamma > 0$, \naiveMulti works until $\gamma$ becomes large enough so
that the estimated confidence interval of $1-5\gamma$ cannot be guaranteed to be
larger than 0.5. Both variants of \textsc{SegCertify} also deteriorate
with increasing noise. The main culprit for this is guessing a wrong class in
the $n_{0}$ samples. We showcase this in \cref{fig:plot_abstain_n_gamma}, where
we use the same setting as in \cref{fig:plot_abstain_gamma} but with
$n_{0} = n = 1000$. Here, both \textsc{SegCertify} variants have the
expected performance (close to $1$) even with increasing noise. For \naiveMulti,
these additional samples also delay complete failure, but
do not prevent it. The main takeaway from
\cref{fig:plot_abstain_n_gamma,fig:plot_abstain_gamma} is that
\textsc{SegCertify} overcomes the bad component issue and the number of
samples $n, n_{0}$ is a key driver for its statistical power.

\input{tab_semanticsegmentation}

Lastly, in \cref{fig:plot_abstain_N}
we compare \textsc{SegCertifyBon} and \textsc{SegCertifyHolm}. Here, $\gamma = 0.05$, $k=0$,
$n = n_{0} = 1000$, $\alpha = 0.1$ and $N$ increases (exponentially) along the
$x$-axis. We see both versions deteriorate in performance as $N$ grows. As
before, this is to be expected as out of $N$ components some of the initial guesses
based on $n_{0}$ samples may be wrong or $n$ may be too small to correctly
reject the test. However, we note that the certification gap between
\textsc{SegCertifyBon} and \textsc{SegCertifyHolm} grows as $N$ increases due to their difference in statistical power.

We conclude that \textsc{SegCertify} addresses both challenges, solving the bad
component issue and achieving better trade-offs in multiple testing.
Particularly in the limit the benefit of Holm correction over Bonferrroni
becomes apparent.

\subsection{Semantic Image Segmentation} \label{sec:eval:semantic}%
Next, we evaluate \textsc{SegCertify} on the task of
semantic image segmentation considering two datasets, Cityscapes, and Pascal
Context.
The Cityscapes dataset \citep{Cordts2016Cityscapes} contains high-resolution ($1024 \times 2048$ pixel)
images of traffic scenes annotated in 30 classes, 19 of which are
used for evaluation. An example of this can be seen in \cref{fig:intro} (the hood
of the car is in an unused class).
The Pascal Context dataset \citep{mottaghi_cvpr14} contains images  with 59 foreground and 1 background classes, which before segmentation are resized to $480 \times 480$. There are two common evaluation schemes,
either using all 60 classes or just the 59 foreground classes. Here we use the
later setting.

As a base model in this settings, we use a HrNetV2
\citep{SunXLW19,WangSCJDZLMTWLX19}. Like many segmentation models, it can be
invoked on different scales. For example, to achieve faster inference we scale
an image down half its length and width, invoke the segmentation model, and
scale up the result. Or, to achieve a more robust or more accurate segmentation
we can invoke segmentation on multiple scales, scale the the output
probabilities to the original input size, average them, and then predict the
class for each pixel. Here we investigate how different scales allow different
trade-offs for provably robust segmentation via \textsc{SegCertify}. We note
that at the end, we always perform certification on the $1024 \times 2048$-scale
result.
We trained our base models with Gaussian Noise ($\sigma = 0.25$),  as in \citet{CohenRK19}. Details about the training procedure and
further results can be found in \cref{sec:details-semantic,sec:additional-results-semantic}.

Evaluation results for 100 images on both datasets are given in
\cref{tab:semanticsegmentation}. We observe that the certified model has accuracy close to that of the
base model, even outperforming it sometimes, even though it is abstaining for a number of pixels.
This means that abstentions generally happens in wrongly predicted areas or on ignored classes.
Similar to \citet{CohenRK19}, we observe best performance
on the noise level we trained the model on ($\sigma = 0.25$) and degrading
performance for higher values. Interestingly, by comparing the results for
$\sigma = 0.5$ at scale 1.0 and smaller scales, we observe that the resizing for
scales $< 1.0$ acts as a natural denoising step, allowing better performance
with more noise.
Further, in \cref{fig:semseg} we show how the certified accuracy degrates with different values of $\tau$.
We use scale $0.25$ and $n=300$ and the same parameters as in \cref{tab:semanticsegmentation} otherwise.
Up to $\tau=0.92$ we observe very gradual drop-off with increasing $\tau$. Then observe a sudden drop-off as $n$ is becomes insufficient to proof any component (even those that are always correct).
All values use Holm correction. We contrast them with Bonferroni correction in \cref{sec:additional-results-semantic}.

\begin{figure}[ht!]
  \centering
  \includegraphics[width=0.8\columnwidth]{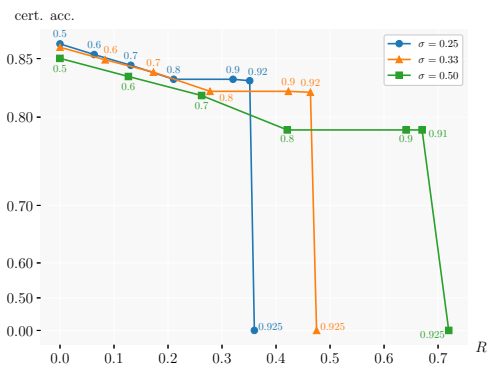}
  \caption{Radius versus certified mean per-pixel accuracy for semantic segmentation on Cityscapes at scale 0.25. Numbers next to dots show $\tau$. The $y$-axis is scaled to the fourth power for clarity.}
  \label{fig:semseg}
\end{figure}

Lastly, we observe that our algorithm comes with a large increase in run-time as
is common with randomized smoothing. However, the effect here is amplified as
semantic segmentation commonly considers large models and large images. On
average, $30$s (for $n=100$) are spent on computing the $p$-values, $0.1$s on
computing the Holm correction and the test of the reported time on sampling. We
did not optimize our implementation for speed, other than choosing the largest
possible batch size for each scale, and we believe that with further engineering
run time can be reduced.

\subsection{Pointcloud Part Segmentation}
\label{sec:pointcl-part-segm}

\input{tab_pointcloud}

\begin{figure}[ht]
  \centering
  \includegraphics[width=0.8\columnwidth]{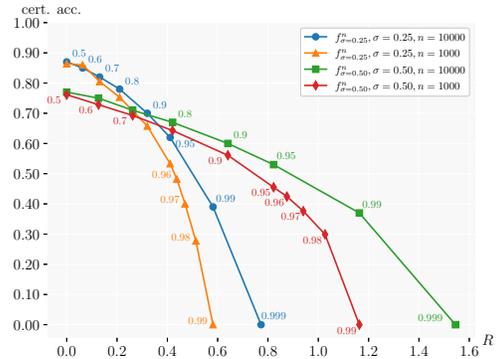}
  \caption{Radius versus certified accuracy at different radii for Pointcloud part segmentation. Numbers next to dots show $\tau$.}
  \label{fig:partseg}
\end{figure}

For point clouds part segmentation we use the ShapeNet
\citep{shapenet2015} dataset, which contains $3d$ CAD models of different
objects from 16 categories, with multiple parts, annotated from 50 labels.
Given the overall label and points sampled from the surface of the object, the
goal is to segment them into their correct parts.
There exist two variations of the described ShapeNet task, one where every point
is a $3d$ coordinate and one where each coordinate also contains its
surface normal vector. Each input consists of 2048 points and before it is
fed to the neural network its coordinates are centered at the origin and scaled
such that the point furthest away from the origin has distance 1.

Using the PointNetV2 architecture \citep{QiSMG17,QiYSG17,yan2020pointasnl}, we
trained four base models: $f_{\sigma=0.25}$, $f_{\sigma=0.25}^{n}$, $f_{\sigma=0.5}^{n}$ and $f_{\sigma=0.5}$ where
$\cdot^{n}$ denotes the inclusion of normal vectors and $\sigma$ the noise used in training.
We apply all smoothing after the rescaling to work on consistent sizes and we
only perturb the location, not the normal vector.
Commonly, the underlying network is invoked 3 times and the results averaged as
the classification algorithm involves randomness. As
\cref{thm:original,thm:fbartau} do not require the underlying $f$ to be
deterministic, we also use the average of $3$ runs for each invocation of the
base model.

We provide the results in \cref{tab:pointcloud,fig:partseg}. We note that on the
same data non-robust model achieves 0.91 and 0.90 with and without normal vectors
respectively. While certification results here are good compared
with the base model, this ratio is worse than in image segmentation and more
samples ($n$) are needed. This is because here -- in contrast to image
segmentation -- a perturbed point can have different semantics if it is moved to
a different part of the object, causing label fluctuation. Further, as
\cref{fig:partseg}, shows we see a gradual decline in accuracy when increasing
$\tau$ rather than the sudden drop in \cref{fig:semseg}. As before, all values
are computed using Holm correction and we provide results for Bonferroni in
\cref{sec:additional-results-pointcl-part-sgm}.

Using the same base models, we apply the randomized smoothing variant by
\citet{FischerBV20} and achieve an accuracy of $0.69$ while showing robustness
to 3d rotations. \cref{sec:app-beyond-ell_p} provides further details and the  obtained guarantee.

\subsection{Discussion \& Limitation}
\label{sec:discussion}
By reducing the problem of certified segmentation to only non-fluctuating components,
we significantly reduce the difficulty and achieve strong results on challenging datasets.
However, a drawback of the method is the newly introduced hyperparameter $\tau$. In practice a
suitable value can be determined by the desired radius or the empirical
performance of the base classifier.
High values of $\tau$ will permit a higher certification radius, but also lead to more abstentions and require more samples $n$ for both certification and inference (both done via \textsc{SegCertify}).
Further, as is common with adversarial defenses is a loss of performance compared to a non-robust model.
However we expect further improvements in accuracy by the application of specialized training procedures
\citep{SalmanLRZZBY19,ZhaiDHZGRH020,JeongS20}, which we briefly
investigate in
\cref{sec:additional-results-semantic}.

%% file: tab_semanticsegmentation.tex
\begin{table*}[ht]
  \centering
  \renewcommand{\arraystretch}{1.2}
  \caption{\footnotesize Segmentation results for 100 images. acc. shows the mean per-pixel accuracy,
    mIoU the mean intersection over union, \%\abstain
    abstentions and $t$ runtime in seconds. All \textsc{SegCertify}
    ($n_{0} = 10, \alpha = 0.001$) results are certifiably robust at radius $R$ w.h.p. multiscale uses $0.5, 0.75, 1.0, 1.25, 1.5, 1.75$ as well as their flipped variants for Cityscapes and additionally $2.0$ for Pascal.
  } \label{tab:semanticsegmentation} \footnotesize
\begin{tabular}{@{}ll@{\hskip 1em}rr@{\hskip 4em}rrrrr@{\hskip 4em}rrrr@{}}
\toprule
  & & & &
  \multicolumn{4}{c}{Cityscapes}  &&
  \multicolumn{4}{c}{Pascal Context}\\
  \cmidrule{5-8}
  \cmidrule{10-13}
  scale & &  $\sigma$ & $R$
  &  acc. &  mIoU & \%\abstain & $t$ &
  &  acc. &  mIoU & \%\abstain & $t$ \\
\midrule
 $0.25$ &
  non-robust model & - & - & 0.93 & 0.60 & 0.00 & 0.38 &
                     & 0.59 &  0.24 & 0.00 &  0.12 \\[0.2em]
  & base model & - & - & 0.87 &  0.42 & 0.00 & 0.37 &
                     & 0.33 & 0.08 & 0.00 &  0.13 \\[0.2em]
  &\multirow{3}{*}{\shortstack[l]{\textsc{SegCertify}\\$n=100,\tau=0.75$}}
                        &  0.25 & 0.17 &  0.84 & 0.43  & 0.07 & 70.00 &&  0.33 & 0.08  &  0.13  &   14.16 \\
    &                    &  0.33 & 0.22 & 0.84  & 0.44  & 0.09 & 70.21 &&  0.34 & 0.09  & 0.17 &   14.20\\
    &                    &  0.50 & 0.34 & 0.82 & 0.43  & 0.13 & 71.45 &&  0.23 &  0.05 &  0.27  &   14.23\\[0.2em]

  &\multirow{3}{*}{\shortstack[l]{\textsc{SegCertify}\\$n=500,\tau=0.95$}}
  & 0.25 & 0.41 & 0.83 & 0.42 & 0.11 & 229.37 && 0.29 &  0.01 & 0.30  &   33.64\\
  &                    & 0.33 & 0.52 & 0.83  & 0.42  & 0.12  & 230.69 && 0.26 &  0.01 &   0.39 &   33.79 \\
  &                   & 0.50 & 0.82 & 0.77  & 0.38  & 0.20 & 230.09 && 0.10 & 0.00 &   0.61 & 33.44\\
\hline \\[-0.6em]
 $0.5$ &
  non-robust model & - & - & 0.96  & 0.76 & 0.00 & 0.39 &
                     & 0.74 & 0.38 & 0.00 &  0.16 \\[0.2em]
  & base model & - & - & 0.89 &  0.51 & 0.00 & 0.39 && 0.47 & 0.13 &   0.00 &    0.14  \\[0.2em]
  &\multirow{3}{*}{\shortstack[l]{\textsc{SegCertify}\\$n=100,\tau=0.75$}}
  & 0.25 & 0.17 & 0.88  & 0.54  & 0.06 & 75.59 &&  0.48  & 0.16 & 0.09 &   16.29 \\
  &                       & 0.33 & 0.22 & 0.87  & 0.54  & 0.08  & 75.99 && 0.50  & 0.17 & 0.11  &   16.08 \\
  &                       & 0.50 & 0.34 & 0.86  & 0.54  & 0.10 & 75.72 && 0.36  & 0.10  & 0.21   &   16.14 \\
  \hline \\[-0.6em]
 $1.0$ &
  non-robust model & - & - & 0.97 & 0.81 & 0.00 & 0.52 &
                     & 0.77 &  0.42 & 0.00 &  0.18 \\[0.2em]
  & base model & - & - & 0.91  & 0.57 & 0.00 & 0.52  && 0.53 &  0.18 & 0.00 & 0.18\\[0.2em]
  &\multirow{3}{*}{\shortstack[l]{\textsc{SegCertify}\\$n=100,\tau=0.75$}}
  & 0.25 & 0.17 & 0.88 & 0.59 & 0.11  & 92.75 && 0.55  & 0.22  &  0.22 &   18.53\\
  &                      & 0.33 & 0.22 & 0.78 & 0.43 & 0.20  & 92.85 && 0.46  & 0.18  &  0.34  &   18.57\\
     &                     & 0.50 & 0.34 & 0.34  & 0.06  & 0.40 & 92.48 && 0.17  & 0.03  &  0.41  &   18.46 \\
  \hline \\[-0.6em]
 multi &
  non-robust model & - & - & 0.97 & 0.82 & 0.00 & 8.98 &
                     & 0.78 & 0.45 & 0.00 & 4.21 \\[0.2em]
  & base model & - & -  & 0.92 & 0.60  & 0.00 & 9.04  && 0.56 &  0.19 &   0.00 & 4.22\\[0.4em]
  &\multirow{1}{*}{\shortstack[l]{\textsc{SegCertify}\\$n=100,\tau=0.75$}}
  &   0.25 & 0.17 & 0.88 & 0.57  & 0.09  & 1040.55 && 0.52  & 0.21  & 0.29  &  355.00 \\[0.8em]
\bottomrule
\end{tabular}
\end{table*}

%% file: tab_pointcloud.tex
\begin{table}[t]
  \centering
  \footnotesize
  \caption{Results on $100$ point clouds part segmentations. acc. shows the mean
    per-point accuracy, \%\abstain abstentions and $t$ runtime in seconds. all
    \textsc{segcertify} ($n_{0} = 100, \alpha = 0.001$) results are certifiably
    robust w.h.p.} \label{tab:pointcloud}
\begin{tabular}{@{}lrrrrrr@{}}
\toprule
 &     $n$ &   $\tau$ & $\sigma$ &   acc &  \%\abstain &     $t$ \\
\midrule
  $f_{\sigma=0.25}$
  & - & - & - & 0.81 &  0.00  &   0.57 \\
 & 1000 &  0.75 & 0.250 &  0.62 & 0.32 &   54.47 \\
 & 1000 &  0.85 & 0.250 &  0.52 & 0.44 &   54.41 \\
 &10000 &  0.95 & 0.250 &  0.41 & 0.56 &   496.57 \\
 & 10000 &  0.99 & 0.250&  0.20 & 0.78&  496.88\\
  \midrule
  $f_{\sigma=0.25}^{n}$
  & - & - & - & 0.86 &  0.00  &   0.57 \\

 & 1000 &  0.75 & 0.250 &  0.78 & 0.16 &   54.46 \\
 & 1000 &  0.85 & 0.250 &  0.71 & 0.25 &   54.41 \\
 &10000 &  0.95 & 0.250 &  0.62 & 0.35 &   496.65 \\
 & 10000 &  0.99 & 0.250&  0.39 &   0.60 &  496.68 \\
  \midrule
  $f_{\sigma=0.5}$
   & - & - & - & 0.70 &   0.00    &   0.55 \\
 & 1000 &  0.75 & 0.250 &  0.57 &   0.23 &   54.40 \\
 & 1000 &  0.75 & 0.500 &  0.47 &   0.44 &   54.47 \\
 & 1000 &  0.85 & 0.500 &  0.41 &   0.54 &   54.39 \\
 &10000 &  0.95 & 0.500 &  0.30 &   0.67 &   496.30 \\
 &10000 &  0.99 & 0.500 &  0.11 &   0.87 &  496.47 \\
  \midrule
  $f_{\sigma=0.5}^{n}$
   & - & - & - & 0.78 &   0.00 & 0.56 \\
 & 1000 &  0.75 & 0.250 &  0.74 &   0.10 & 54.78 \\
 & 1000 &  0.75 & 0.500 &  0.67 &   0.21 & 54.58 \\
 & 1000 &  0.85 & 0.500 &  0.61 &   0.30 & 54.44 \\
 &10000 &  0.95 & 0.500 &  0.53 &   0.41 & 497.40 \\
 &10000 &  0.99 & 0.500 &  0.37 &   0.60 & 497.87 \\

\bottomrule
\end{tabular}
\end{table}

%% file: conclusion.tex
\section{Conclusion}
\label{sec:conclusion}

In this work we investigated the problem of provably robust segmentation algorithms and
identified two key challenges: bad components and trade-offs introduced by
multiple testing that prevent naive solutions from working well. Based on this we
introduced \textsc{SegCertify}, a simple yet powerful algorithm that clearly
overcomes the bad component issue and allows for significantly better trade-offs
in multiple-testing. It enables certified performance on a similar level as an
undefended base classifier and permits guarantees to multiple threat models in a
plug-and-play manner.

%% file: appendix.tex
\twocolumn[
\icmltitle{Supplemental Material for\\ \titlet}
\vskip 0.3in
]

\input{appendix_exp_details}

\input{appendix_additional_results}

\input{tab_semnaticsegmentation_extended}

\clearpage
\input{appendix_attack_vis}

%% file: appendix_exp_details.tex
\section{Experimental Details}
\label{sec:app-experiments}

\subsection{Experimental Details for \cref{sec:eval:semantic}}
\label{sec:details-semantic}
We use a HrNetV2 \citep{SunXLW19,WangSCJDZLMTWLX19} with the HRNetV2-W48
backbone from their official PyTorch 1.1 \citep{PaszkeGMLBCKLGA19}
implementation\footnote{\url{https://github.com/HRNet/HRNet-Semantic-Segmentation}}.
For Cityscapes we follow the outlined training procedure, only adding
the $\sigma=0.25$ Gaussian noise. For Pascal Context we doubled the number of
training epochs (and learning rate schedule) and added the $\sigma=0.25$ Gaussian noise.
During inference we use different batch sizes for different scales. These are summarized in \cref{tab:batch_sizes}.
All timing timing results are given for a single Nvidia GeForce RTX 2080 Ti and using 12 cores of a Intel(R) Xeon(R) Silver 4214R CPU @ 2.40GHz.
When training on an machine with 8 Nvidia GeForce RTX 2080 Ti one training epoch
takes around 4 minutes for both of the data sets.

We evaluate on 100 images each, that is for Cityscapes we use every \nth{5} image in the test set and for Pascal every \nth{51}.

We consider two metrics:

\begin{itemize}
  \item (certified) per-pixel accuracy: the rate of pixels correctly classified (over all images)
  \item (certified) mean intersection over union (mIoU): For each image $i$ and each class $c$ (ignoring the \abstain ``class'') we calculate the ratio $IoU_{c}^{i} = \tfrac{|P_{c}^{i} \cap G_{c}^{i}|}{|P_{c}^{i} \cup G_{c}^{i}|}$ where $P_{c}^{i}$ denotes the pixel locations predicted as class $c$ for input $i$ and $G_{c}^{i}$ denotes the pixel locations for class $c$ in the ground truth of input $i$. We then average $IoU_{c}^{i}$ over all inputs and classes.
\end{itemize}

\begin{table}[t]
  \centering
   \caption{Batch sizes used in segmentation inference.}
  \label{tab:batch_sizes}
 \begin{tabular}{lcc}
   \toprule
    scale & Cityscapes & Pascal Context\\
   \midrule
    0.25 & 24 & 80\\
    0.50 & 12 & 64\\
    0.75 & 4 & 32\\
    1.00 & 4 & 20\\
    multi & 4 & 10\\
   \bottomrule
  \end{tabular}
\end{table}

\subsection{Experimental Details for \cref{sec:pointcl-part-segm}}
\label{sec:details-pointcl-partsegm}

Using the PointNetV2 architecture \citep{QiSMG17,QiYSG17,yan2020pointasnl}
implemented in PyTorch
\footnote{\url{https://github.com/yanx27/Pointnet_Pointnet2_pytorch}}
\citep{PaszkeGMLBCKLGA19}.
Again we keep the training parameters unchanged other than the addition of noise during training. One training epoch on a single Nvidia GeForce RTX 2080 Ti takes 6 minutes.
In inference we use a batch size of 50.
All timing results are given for a single Nvidia GeForce RTX 2080 Ti and using 12 cores of a Intel(R) Xeon(R) Silver 4214R CPU @ 2.40GHz.

Again, we evaluate on 100 inputs. This corresponds to every \nth{28} input in the test set.
As a metric we consider the (certified) per-component accuracy: the rate of parts/components correctly classified (over all inputs).

%% file: appendix_additional_results.tex
\section{Additional Results}

\subsection{Additional Results for \cref{sec:synthetic}}
\label{sec:additional-results-synthetic}

In \cref{fig:plot_abstain_n_gamma,fig:plot_abstain_N} we show smoothed plots as
\naiveCartesian and \naiveMulti are either correct on all components or non at
all. Here, we provide the unsmoothed results in \cref{fig:plot_abstain_app}. In order
to obtain the plots in \cref{fig:plot_abstain} we apply a Savgol filter
\citep{savitzky1964smoothing} of degree 1 over the 11 closest neighbours (using
the SciPy implementation) and use a step size of $0.001$ for $\gamma$.

\begin{figure*}[h]
  \centering
  \newcommand{\subfigwidth}{0.45\textwidth}
    \begin{subfigure}[b]{\subfigwidth}
        \includegraphics[width=\textwidth]{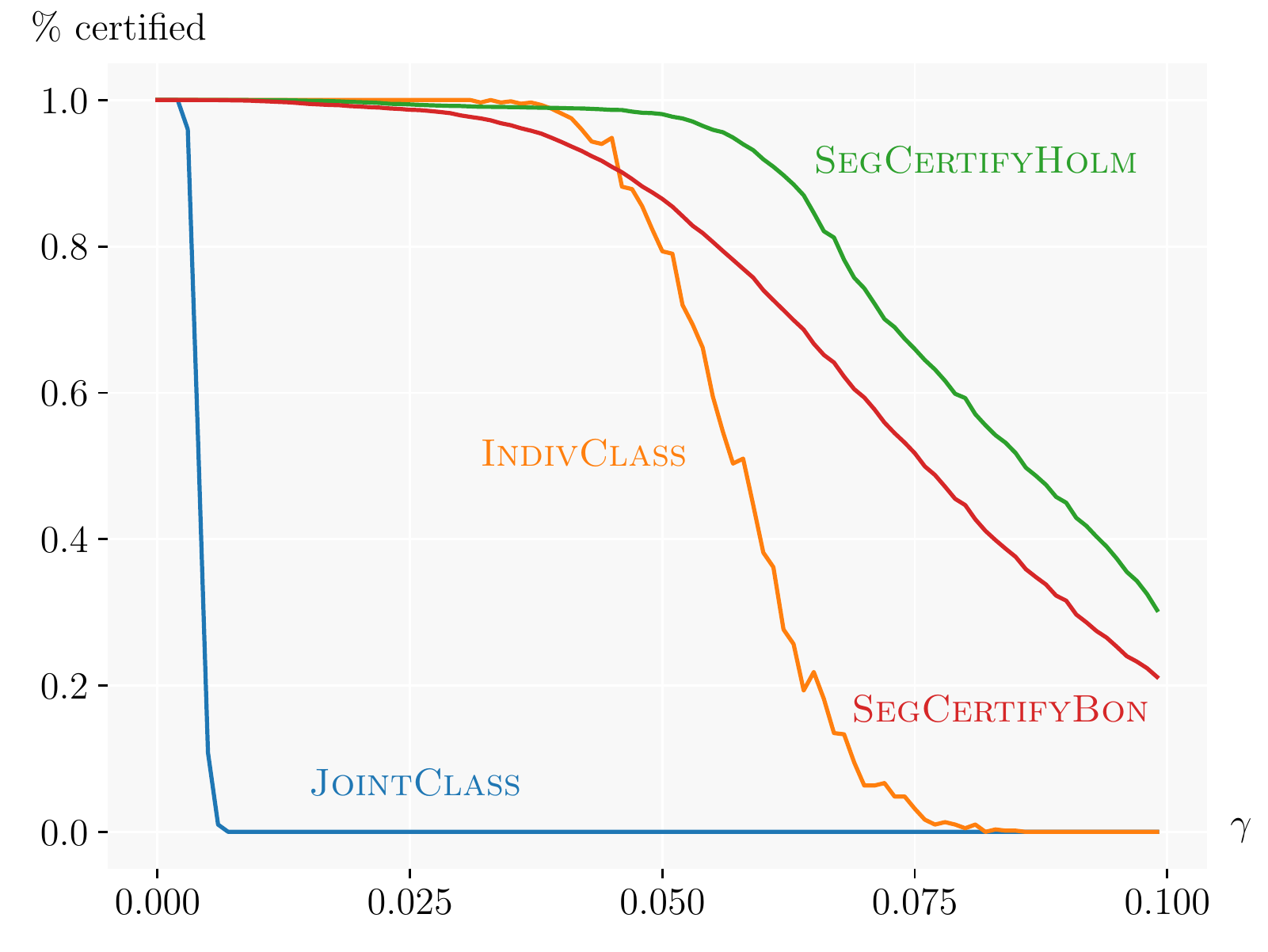}
      \caption{Unsmoothed version of \cref{fig:plot_abstain_n_gamma}. On $N=100$ components, with a classifier that has error rate $5\gamma$ on one component and $\gamma$ on all others.}
        \label{fig:plot_abstain_gamma_app}
    \end{subfigure}
    \hfill
    \begin{subfigure}[b]{\subfigwidth}
\includegraphics[width=\textwidth]{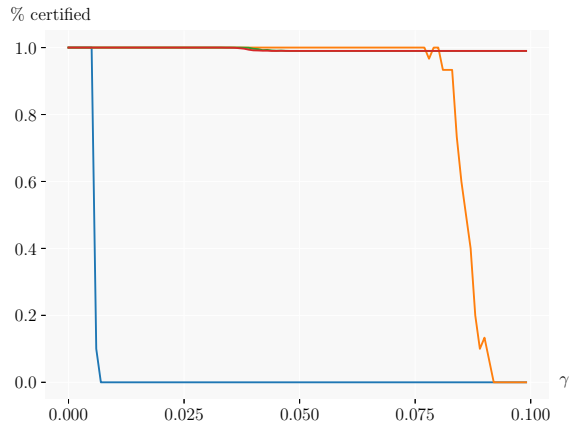}
      \caption{Unsmoothed version of \cref{fig:plot_abstain_N}.
\textsc{SegCertify} with different testing corrections for various numbers of components $N$ with error rate $\gamma = 0.05$.}
      \label{fig:plot_abstain_n_app}
    \end{subfigure}
    \caption{Unsmoothed versions of \cref{fig:plot_abstain}.
      We empirically investigate the power (ability to avoid type II
      errors -- false abstention) of multiple algorithms on synthetic data. The
      $y$-axis shows the rate of certified (rather than abstained) components. An optimal algorithm would achieve
      $1.0$ or $0.99$ in all plots.
     }
  \label{fig:plot_abstain_app}
\end{figure*}

\begin{figure*}[ht!]
  \centering
  \newcommand{\subfigwidth}{0.45\textwidth}
  \begin{subfigure}[b]{\subfigwidth}
    \includegraphics[width=\textwidth]{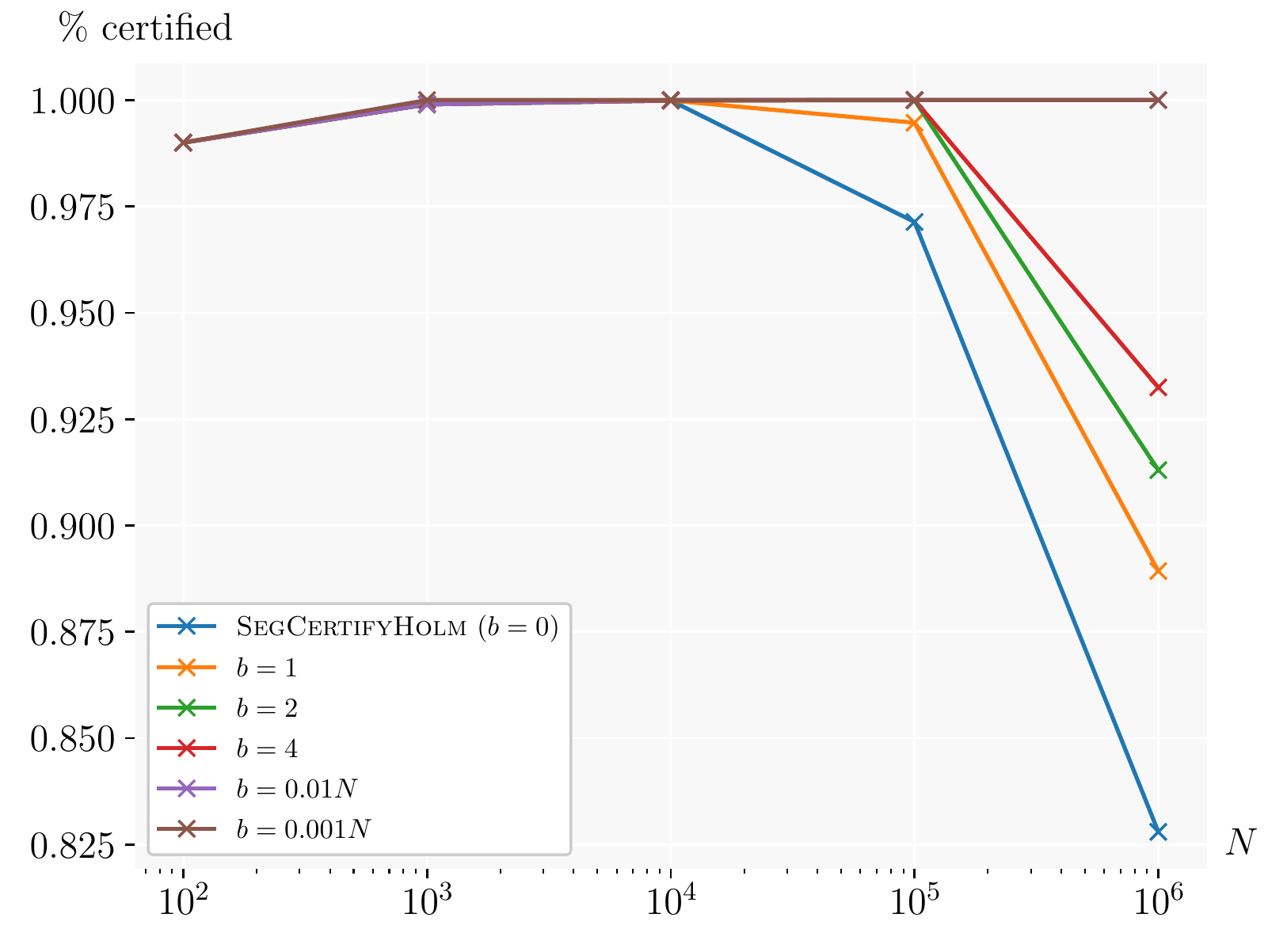}
    \caption{$\alpha = 0.1$.}
    \label{fig:plot_abstain_N_kfwer}
  \end{subfigure}
    \hfill
    \begin{subfigure}[b]{\subfigwidth}
      \includegraphics[width=\textwidth]{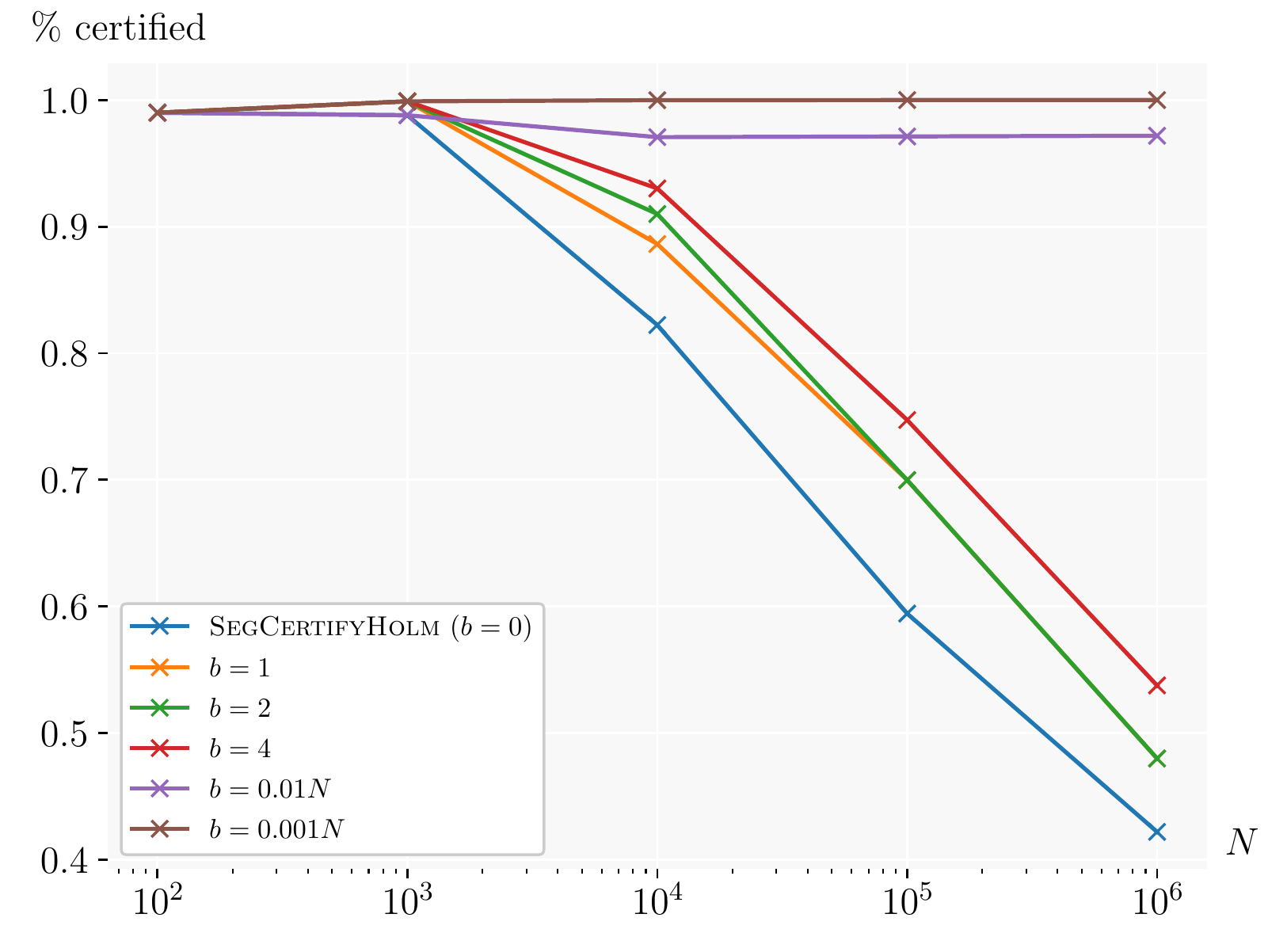}
      \caption{$\alpha = 0.001$.}
        \label{fig:plot_abstain_N_kfwer0001}
    \end{subfigure}
    \caption{Evaluation where an error budget of up top $b$ type I errors is
      allowed. Potentially allowing a small amount of errors greatly increases
      the power of the test. $N$ varies along the $x$-axis, $\gamma=0.05$ and $k$
      components with error rate $5\gamma$.}
  \label{fig:plot_abstain_kfwer}
\end{figure*}

\input{appendix_kfwer}

\subsection{Additional Results for \cref{sec:eval:semantic}}
\label{sec:additional-results-semantic}

\input{tab_semanticsegmentation_bon}

\cref{tab:semanticsegmentation_extended} shows an extended version of \cref{tab:semanticsegmentation}. Both of these tables use Holm correction.
\cref{tab:semanticsegmentation_bon} shows the difference to  \cref{tab:semanticsegmentation_extended} if instead Bonferroni correction was used.
Generally this difference is very small in this setting as it appears the true $p_{A}$ are far from $\tau$. However in some settings (such as multi resolution) the effect of Holm correction can be up to 2\%.
Since for a particular base classifier or $\tau$ this gain can quickly go from neglectable to significant and since the additional evaluation time ($< 0.1s$) is neglectable compared to the time for sampling we believe Holm correction to be preferable in most cases.

\paragraph{Consistency Training}
Here we investigate a naive instantiation of training approach from \citet{JeongS20}.

\citet{JeongS20} improve the training for classification models used as base models in randomized smoothing by adding a consistency regularization term in training. We use this same term, but compute it for every pixel and average the results.
To obtain the results in \cref {tab:semanticsegmentation_consistency} we used $m=2$, $\lambda = 1$, $\eta = 0.5$ and $\sigma = 0.25$.
Depending on the scale we see either slightly better or slightly worse results than with standard Gaussian data augmentation in training. This shows the promise of the method but also highlights the need for potential further specialization to the segmentation setting, particularly by considering the effect of scaling.

\input{tab_semanticsegmentation_consitency}

\subsection{Additional Results for \cref{sec:pointcl-part-segm}}
\label{sec:additional-results-pointcl-part-sgm}

\input{tab_pointcloud_bon}

\cref{tab:pointcloud_bon} shows the change when executing the experiments in \cref{tab:pointcloud} with Bonferroni correction instead of Holm correction.

\input{appendix_rotation}

%% file: appendix_kfwer.tex
\subsection{$k$-FWER and error budget}
\label{sec:app-kfwer}

Here we discuss the gains from allowing a small budget of errors
and applying $k$-FWER control as outlined in \cref{sec:extensions}. Control for
$k$-FWER at level $\alpha$ means that $P(\geq k \text{ type I errors}) \leq \alpha$. Which for
$k=1$ recovers standard FWER control. Thus, if we allow a budget of $b$ type I
errors at level $\alpha$ we need to perform $k$-FWER control with $k = b+1$. In the
following we will refer only to the budget $b$, to avoid confusion between the
$k$ in $k$-FWER and the $k$ noisy components in the setting of
\cref{sec:synthetic}.

\cref{fig:plot_abstain_kfwer} shows an empirical evaluation of this approach for
different $b$ for $\gamma = 0.05$, one noisy components and different levels of
$N$. \cref{fig:plot_abstain_N_kfwer} uses $\alpha = 0.1$ and
\cref{fig:plot_abstain_N_kfwer0001} $\alpha = 0.001$. We see that allowing a single
error leads to a huge gain in power for the $N = 10^{6}$ setting. Similarly,
allowing 1 percent or 1 permille errors greatly strengthen the method.

\paragraph{False Discovery Rate}
Similarly, \textsc{SegCertify} can employ false discovery rate
(FDR) control rather than ($k$-)FWER control. FDR control limits the expected
number of type I errors. Since this is a much weaker statement than FWER it
allows for less type II errors. However, while useful this kind of control
leaves the area of (statistical) certified robustness for a more relaxed
probabilistic setting which we do not investigate here further.

%% file: tab_semanticsegmentation_bon.tex
\begin{table*}
  \centering
  \renewcommand{\arraystretch}{1.2}
  \footnotesize
  \caption{Difference when Bonferroni correction rather than Holm correction is used in    \cref{tab:semanticsegmentation_extended}. Only differences $\geq 10^{-4}$ are shown. We observed no such differences on the Pascal Context dataset.}
  \vspace{1em}
 \label{tab:semanticsegmentation_bon}
\begin{tabular}{@{}ll@{\hskip 1em}rr@{\hskip 4em}rrr@{}}
\toprule
  & & & &
          \multicolumn{3}{c}{Cityscapes}\\
  \cmidrule{5-7}
  scale & &  $\sigma$ & $R$
        &  acc. &  mIoU & \%\abstain\\
\midrule

 $0.5$ &
  &\multirow{3}{*}{\shortstack[l]{\textsc{SegCertify}\\$n=100,\tau=0.75$}}
    0.25 & 0.17 & -0.0008 & -0.0005 & 0.0019 \\
    & &0.33 & 0.22 & -0.0014 & -0.0010 & 0.0024 \\
    & &0.50 & 0.34 & -0.0021 & -0.0019 & 0.0031 \\
  \hline \\[-0.6em]

 $0.5$ &
  &\multirow{3}{*}{\shortstack[l]{\textsc{SegCertify}\\$n=100,\tau=0.75$}}
    0.25 & 0.17 & -0.0009 & -0.0014 & 0.0014  \\
    & &0.33 & 0.22 & -0.0012 & -0.0020 & 0.0016  \\
    & &0.50 & 0.34 & -0.0005 & -0.0004 & 0.0006  \\
  \hline \\[-0.6em]

 $1.0$ &
  &\multirow{3}{*}{\shortstack[l]{\textsc{SegCertify}\\$n=100,\tau=0.75$}}
    0.25 & 0.17 & -0.0012 & -0.0032 & 0.0016  \\
    & &0.33 & 0.22 & -0.0020 & -0.0018 & 0.0024 \\
    & &0.50 & 0.34 & -0.0000 & -0.0000 & 0.0001  \\[0.4em]
  & &\shortstack[l]{\textsc{SegCertify}\\$n=300,\tau=0.90$}
    0.25 & 0.17 & 0.0001 & 0.0002 & -0.0002 \\
  \hline \\[-0.6em]

 multi &
  &\shortstack[l]{\textsc{SegCertify}\\$n=100,\tau=0.75$}
    0.25 & 0.17 & -0.0125 & -0.0185 & 0.0173 \\

\bottomrule
\end{tabular}
\end{table*}

%% file: tab_semanticsegmentation_consitency.tex
\begin{table*}
  \centering
  \renewcommand{\arraystretch}{1.2}
  \footnotesize
  \caption{Same setting as \cref{tab:semanticsegmentation_extended} but using a model trained with consitency regularization.}
  \vspace{1em}
  \label{tab:semanticsegmentation_consistency}
\begin{tabular}{@{}ll@{\hskip 1em}rr@{\hskip 4em}rrr@{}}
\toprule
  & & & &
          \multicolumn{3}{c}{Cityscapes}\\
  \cmidrule{5-7}
  scale & &  $\sigma$ & $R$
  &  acc. &  mIoU & \%\abstain\\
\midrule
 $0.25$ &
  base model & - & - & 0.87 & 0.40  & 0.00\\[0.2em]
  &\multirow{3}{*}{\shortstack[l]{\textsc{SegCertify}\\$n=100,\tau=0.75$}}
                        &  0.25 & 0.17 &  0.83 &  0.42 & 0.07 \\
    &                    &  0.33 & 0.22 & 0.84 &  0.42 & 0.09 \\
    &                    &  0.50 & 0.34 & 0.82 &  0.43 & 0.14 \\[0.2em]
 \hline \\[-0.6em]

 $0.50$ &
  base model & - & - & 0.91 &  0.53 & 0.00\\[0.2em]
  &\multirow{3}{*}{\shortstack[l]{\textsc{SegCertify}\\$n=100,\tau=0.75$}}
                        &  0.25 & 0.17 &  0.89 &  0.57 & 0.06 \\
    &                    &  0.33 & 0.22 & 0.89 &  0.57 & 0.07 \\
    &                    &  0.50 & 0.34 & 0.86 &  0.57 & 0.11 \\[0.2em]
 \hline \\[-0.6em]

 $1.00$ &
  base model & - & - & 0.92 & 0.62 & 0.00\\[0.2em]
  &\multirow{3}{*}{\shortstack[l]{\textsc{SegCertify}\\$n=100,\tau=0.75$}}
                        &  0.25 & 0.17 &  0.88 &  0.63 & 0.12 \\
    &                    &  0.33 & 0.22 & 0.79 &  0.46 & 0.20 \\
  &                    &  0.50 & 0.34 & 0.39 &  0.06 & 0.32 \\

\bottomrule
\end{tabular}

\end{table*}

%% file: tab_pointcloud_bon.tex
\begin{table*}[h]
  \centering
  \footnotesize
  \caption{Difference when Bonferroni correction rather than Holm correction is used in    \cref{tab:pointcloud}.} 
  \vspace{1em}
  \label{tab:pointcloud_bon}
\begin{tabular}{@{}lrrrrrr@{}}
\toprule
 &     $n$ &   $\tau$ & $\sigma$ &   acc &  \%\abstain \\
\midrule
  $f_{\sigma=0.25}$\\
 & 1000 &  0.75 & 0.250 & -0.0012  & 0.0019 \\
 & 1000 &  0.85 & 0.250 & -0.0023  & 0.0029 \\
 &10000 &  0.95 & 0.250 & -0.0013  & 0.0009 \\
 & 10000 &  0.99 & 0.250& -0.0008  & 0.0009 \\
  \midrule
  $f_{\sigma=0.25}^{n}$\\
 & 1000 &  0.75 & 0.250 &  -0.0020 & 0.0026 \\
 & 1000 &  0.85 & 0.250 &  -0.0026 & 0.0032 \\
 &10000 &  0.95 & 0.250 &  -0.0011 & 0.0010 \\
 & 10000 &  0.99 & 0.250&  -0.0013 & 0.0011 \\
  \midrule
  $f_{\sigma=0.5}$ \\
 & 1000 &  0.75 & 0.250 &  -0.0007 & 0.0010 \\
 & 1000 &  0.75 & 0.500 &  -0.0002 & 0.0010 \\
 & 1000 &  0.85 & 0.500 &  -0.0017 & 0.0028 \\
 &10000 &  0.95 & 0.500 &  -0.0007 & 0.0010 \\
 &10000 &  0.99 & 0.500 &  -0.0003 & 0.0003 \\
  \midrule
  $f_{\sigma=0.5}^{n}$ \\
 & 1000 &  0.75 & 0.250 & -0.0009  & 0.0017  \\
 & 1000 &  0.75 & 0.500 & -0.0015  & 0.0024  \\
 & 1000 &  0.85 & 0.500 & -0.0019  & 0.0028  \\
 &10000 &  0.95 & 0.500 & -0.0043  & 0.0013  \\
 &10000 &  0.99 & 0.500 & -0.0008  & 0.0007  \\

\bottomrule
\end{tabular}
\end{table*}

%% file: appendix_rotation.tex
\subsection{Certification beyond $\ell_p$}
\label{sec:app-beyond-ell_p}

As outlined in \cref{sec:extensions}, \textsc{SegCertify} can be easily adapted
to non-$\ell_{2}$-settings. Here we show that we can certify against and adversary
rotating 3d point clouds. A 3d rotation is parameterized by 3 angles which we
will denote $\epsilon \in \mathbb{R}^{3}$ and define
$\psi_{\epsilon}(\vx) \colon \mathbb{R}^{N \times 3} \to \mathbb{R}^{N \times 3}$ as
\begin{equation}
  \label{eq:rot}
  (\psi_{\epsilon}(\vx))_{i} = \mR_{\epsilon} \vx_{i},
\end{equation}

where $\mR_{\epsilon}$ denotes the 3d rotation matrix specified by $\epsilon$. The randomized
smoothing approach of \citet{FischerBV20} allows to certify robustness in this
case: $f(\vx) = f(\psi_{\delta}(\vx))$ for $\delta$ with $\|\delta\|_{2} \leq R$, by sampling
rotations $\psi_{\epsilon}(\vx)$ with $\epsilon \sim \mathcal{N}(0, \sigma^{2})$. The parameter
robustness radius $R $ is computed the same as throughout the paper. When
applied to points with a normal vector, \cref{eq:rot} can be extended to apply
$R_{\epsilon}$ to the point coordinates as well as the normal vector.

Using one of the model from \cref{sec:pointcl-part-segm}, $f_{\sigma=0.5}^{n}$ ,
we perform this version of randomized smoothing.

The results are shown in \cref{tab:result_rot}. Since the models was not
specifically trained to be robust under rotations, the performance quickly
deteriorates. Nevertheless we can certify robustness to rotations with a
parameter radius $R$ of 0.17 and 0.085 for $\sigma$ of $0.25$ and $0.125$
respectively.

The same approach can be applied to models that are
empirically invariant to most rotations while not formally rotation invariant.
In these cases we need to certify a radius of
$R = \sqrt{3} \pi$ (when measuring angles in radians). When using a fixed $\tau$, an
appropriate $\sigma$ can be chosen as $\sigma = \frac{\sqrt{3} \pi}{\Phi^{-1}(\tau)}$. While this
is relativity large $\sigma$, this does not pose an obstacle for a mostly robust base
model.

\begin{minipage}{\columnwidth}
  \centering
  \footnotesize
  \captionof{table}{Results for point cloud part segmentation under 3d rotation. The
    baseline and base model is $f^{n}_{\sigma=0.5}$. \textsc{SegCertify} uses
    $\tau =0.75$, $n_{0}=100$, $n=1000$ and $\alpha=0.0001$.}
  \label{tab:result_rot}
  \vspace{0.5em}
 \begin{tabular}{@{}rrrr@{}}
   \toprule
    model~/~$\sigma$ & \text{acc} & \%\abstain & $t$\\
    \midrule
baseline & 0.77 & 0.00 & 0.72\\
0.125 & 0.69 & 0.16 & 74.13\\
0.25 & 0.61 & 0.26 & 74.51\\
   \bottomrule
  \end{tabular}
\end{minipage}

%% file: tab_semnaticsegmentation_extended.tex
\begin{table*}[p]
  \centering
  \renewcommand{\arraystretch}{1.1}
  \footnotesize
  \caption{Extended version of \cref{tab:semanticsegmentation}. Segmentation
    results for 100 images. acc. shows the mean per-pixel accuracy, mIoU the
    mean intersection over union, \%\abstain abstentions and $t$ runtime in
    seconds. All \textsc{SegCertify} ($n_{0} = 10, \alpha = 0.001$) results are
    certifiably robust at radius $R$ w.h.p. multiscale uses
    $0.5, 0.75, 1.0, 1.25, 1.5, 1.75$ as well as their flipped variants for
    Cityscapes and additionally $2.0$ for Pascal. All numbers are obtained via Holm correction.
  } 
  \vspace{1em}
  \label{tab:semanticsegmentation_extended}
\begin{tabular}{@{}ll@{\hskip 1em}rr@{\hskip 4em}rrrrr@{\hskip 4em}rrrr@{}}
\toprule
  & & & &
  \multicolumn{4}{c}{Cityscapes}  &&
  \multicolumn{4}{c}{Pascal Context}\\
  \cmidrule{5-8}
  \cmidrule{10-13}
  scale & &  $\sigma$ & $R$
  &  acc. &  mIoU & \%\abstain & $t$ &
  &  acc. &  mIoU & \%\abstain & $t$ \\
\midrule

 $0.25$ &
  non-robust model & - & - & 0.93 & 0.60 & 0.00 & 0.38 &
                     & 0.59 &  0.24 & 0.00 &  0.12 \\[0.2em]
  & base model & - & - & 0.87 &  0.42 & 0.00 & 0.37 &
                     & 0.33 & 0.08 & 0.00 &  0.13 \\[0.2em]
  &\multirow{3}{*}{\shortstack[l]{\textsc{SegCertify}\\$n=100,\tau=0.75$}}
                        &  0.25 & 0.17 &  0.84 & 0.43  & 0.07 & 70.00 &&  0.33 & 0.08  &  0.13  &   14.16 \\
    &                    &  0.33 & 0.22 & 0.84  & 0.44  & 0.09 & 70.21 &&  0.34 & 0.09  & 0.17 &   14.20\\
    &                    &  0.50 & 0.34 & 0.82 & 0.43  & 0.13 & 71.45 &&  0.23 &  0.05 &  0.27  &   14.23\\[0.2em]

    &\multirow{3}{*}{\shortstack[l]{\textsc{SegCertify}\\$n=300,\tau=0.90$}}
                        & 0.25 & 0.32 &  0.83 &  0.43 & 0.10 & 143.37 && 0.32 &  0.08&   0.23 &   24.33\\
  &                     & 0.33 & 0.42 &  0.82 &  0.43 & 0.12 & 143.30 && 0.32 &  0.09 &  0.29&   24.42\\
   &                    & 0.50 & 0.64 & 0.79 &  0.40 & 0.18 & 143.54 && 0.20&  0.04&   0.43&   24.13\\

  &\multirow{3}{*}{\shortstack[l]{\textsc{SegCertify}\\$n=500,\tau=0.95$}}
  & 0.25 & 0.41 & 0.83 & 0.42 & 0.11 & 229.37 && 0.29 &  0.01 & 0.30  &   33.64\\
  &                    & 0.33 & 0.52 & 0.83  & 0.42  & 0.12  & 230.69 && 0.26 &  0.01 &   0.39 &   33.79 \\
  &                   & 0.50 & 0.82 & 0.77  & 0.38  & 0.20 & 230.09 && 0.10 & 0.00 &   0.61 & 33.44\\

  &\multirow{3}{*}{\shortstack[l]{\textsc{SegCertify}\\$n=10000,\tau=0.99$}}
                        &  0.25 & 0.58 &  - & -  & - & - && 0.25 &  0.07 &   0.48 &    557.29\\
   &                     &  0.33 & 0.77 &  - & -  & - & - && 0.24 & 0.07 & 0.58 & 557.34 \\
   &                     &  0.50 & 1.17 &  - & -  & - & - && 0.11 & 0.03 & 0.77 & 557.32\\[0.4em]
  \hline \\[-0.6em]

 $0.5$ &
  non-robust model & - & - & 0.96  & 0.76 & 0.00 & 0.39 &
                     & 0.74 & 0.38 & 0.00 &  0.16 \\[0.2em]
  & base model & - & - & 0.89 &  0.51 & 0.00 & 0.39 && 0.47 & 0.13 &   0.00 &    0.14  \\[0.2em]
  &\multirow{3}{*}{\shortstack[l]{\textsc{SegCertify}\\$n=100,\tau=0.75$}}
  & 0.25 & 0.17 & 0.88  & 0.54  & 0.06 & 75.59 &&  0.48  & 0.16 & 0.09 &   16.29 \\
  &                       & 0.33 & 0.22 & 0.87  & 0.54  & 0.08  & 75.99 && 0.50  & 0.17 & 0.11  &   16.08 \\
  &                       & 0.50 & 0.34 & 0.86  & 0.54  & 0.10 & 75.72 && 0.36  & 0.10  & 0.21   &   16.14 \\

   &\multirow{3}{*}{\shortstack[l]{\textsc{SegCertify}\\$n=300,\tau=0.90$}}
                          & 0.25 & 0.32 & 0.87 & 0.53 & 0.08 & 143.40 && 0.46 &  0.15 &   0.17 &  27.17 \\
  &                       & 0.33 & 0.42 & 0.86 & 0.52 & 0.10 & 145.90 && 0.47 & 0.16 &   0.21 &   27.17 \\
  &                       & 0.50 & 0.64 & 0.83 & 0.50 & 0.15 & 144.61 && 0.31 &  0.10 &   0.38 &   27.32\\[0.4em]
   &\multirow{3}{*}{\shortstack[l]{\textsc{SegCertify}\\$n=500,\tau=0.95$}}
                          & 0.25 & 0.41 & 0.86 & 0.52 & 0.09 & 228.63 && 0.45 & 0.19  & 0.21  & 38.27\\
  &                       & 0.33 & 0.52 & 0.85 & 0.51 & 0.11 & 228.38 &&  0.46 & 0.16  & 0.26 & 38.43\\
  &                       & 0.50 & 0.82 & 0.82 & 0.49 & 0.16 & 228.73 && 0.30 & 0.09 & 0.44 & 38.37\\
  \hline \\[-0.6em]

  $0.75$ &
non-robust model & - & - & 0.97 & 0.80 & 0.00 & 0.46 && 0.76 &  0.41 &   0.00 &   0.15  \\[0.2em]
&base model & - & - & 0.90 & 0.59 & 0.00 & 0.47 && 0.55 &  0.18 &   0.00 &    0.15  \\[0.2em]
   &\multirow{3}{*}{\shortstack[l]{\textsc{SegCertify}\\$n=100,\tau=0.75$}}
                          & 0.25 & 0.17 & 0.88 & 0.57 & 0.09 & 82.69 && 0.53 &  0.19 &   0.15 &   16.78 \\
  &                       & 0.33 & 0.22 & 0.86 & 0.56 & 0.12 & 82.87 && 0.54 &  0.20 &   0.20 &   16.83 \\
  &                       & 0.50 & 0.34 & 0.64 & 0.27 & 0.31 & 82.19 && 0.29 &  0.07 &   0.33 &   16.83\\[0.4em]
    &\multirow{3}{*}{\shortstack[l]{\textsc{SegCertify}\\$n=300,\tau=0.90$}}
                          & 0.25 & 0.32 & 0.84 & 0.54 & 0.13 & 177.44 && 0.51 &  0.19 &   0.22 &   29.48 \\
  &                       & 0.33 & 0.42 & 0.84 & 0.52 & 0.15 & 177.22 && 0.51 &  0.20 &   0.28 &   29.58 \\
  &                       & 0.50 & 0.64 & 0.60 & 0.24 & 0.37 & 177.67 && 0.25 &  0.06 &   0.45 &   29.53 \\
 \hline \\[-0.6em]

   $1.0$ &
  non-robust model & - & - & 0.97 & 0.81 & 0.00 & 0.52 &
                     & 0.77 &  0.42 & 0.00 &  0.18 \\[0.2em]
  & base model & - & - & 0.91  & 0.57 & 0.00 & 0.52  && 0.53 &  0.18 & 0.00 & 0.18\\[0.2em]
  &\multirow{3}{*}{\shortstack[l]{\textsc{SegCertify}\\$n=100,\tau=0.75$}}
  & 0.25 & 0.17 & 0.88 & 0.59 & 0.11  & 92.75 && 0.55  & 0.22  &  0.22 &   18.53\\
  &                      & 0.33 & 0.22 & 0.78 & 0.43 & 0.20  & 92.85 && 0.46  & 0.18  &  0.34  &   18.57\\
     &                     & 0.50 & 0.34 & 0.34  & 0.06  & 0.40 & 92.48 && 0.17  & 0.03  &  0.41  &   18.46 \\

    &\multirow{3}{*}{\shortstack[l]{\textsc{SegCertify}\\$n=300,\tau=0.90$}}
                          & 0.25 & 0.32 & 0.86 & 0.56 & 0.14 & 204.82 && 0.53 &  0.21 &   0.29 &   33.83 \\
  &                       & 0.33 & 0.42 & 0.75 & 0.40 & 0.24 & 204.58 && 0.42 &  0.17 &   0.44 &   33.78 \\
  &                       & 0.50 & 0.64 & 0.31 & 0.05 & 0.47 & 204.57 && 0.15 &  0.03 &   0.52 &   33.43 \\
  \hline \\[-0.6em]

 multi &
  non-robust model & - & - & 0.97 & 0.82 & 0.00 & 8.98 &
                     & 0.78 & 0.45 & 0.00 & 4.21 \\[0.2em]
  & base model & - & -  & 0.92 & 0.60  & 0.00 & 9.04  && 0.56 &  0.19 &   0.00 & 4.22\\[0.4em]
  &\multirow{1}{*}{\shortstack[l]{\textsc{SegCertify}\\$n=100,\tau=0.75$}}
  &   0.25 & 0.17 & 0.88 & 0.57  & 0.09  & 1040.55 && 0.52  & 0.21  & 0.29  &  355.00 \\[0.8em]
\bottomrule
\end{tabular}
\end{table*}

%% file: appendix_attack_vis.tex
\begin{figure*}[ht!]
    \newcommand{\subfigwidth}{0.4\columnwidth}
    \center
    \begin{subfigure}[b]{\subfigwidth}
        \includegraphics[width=\textwidth]{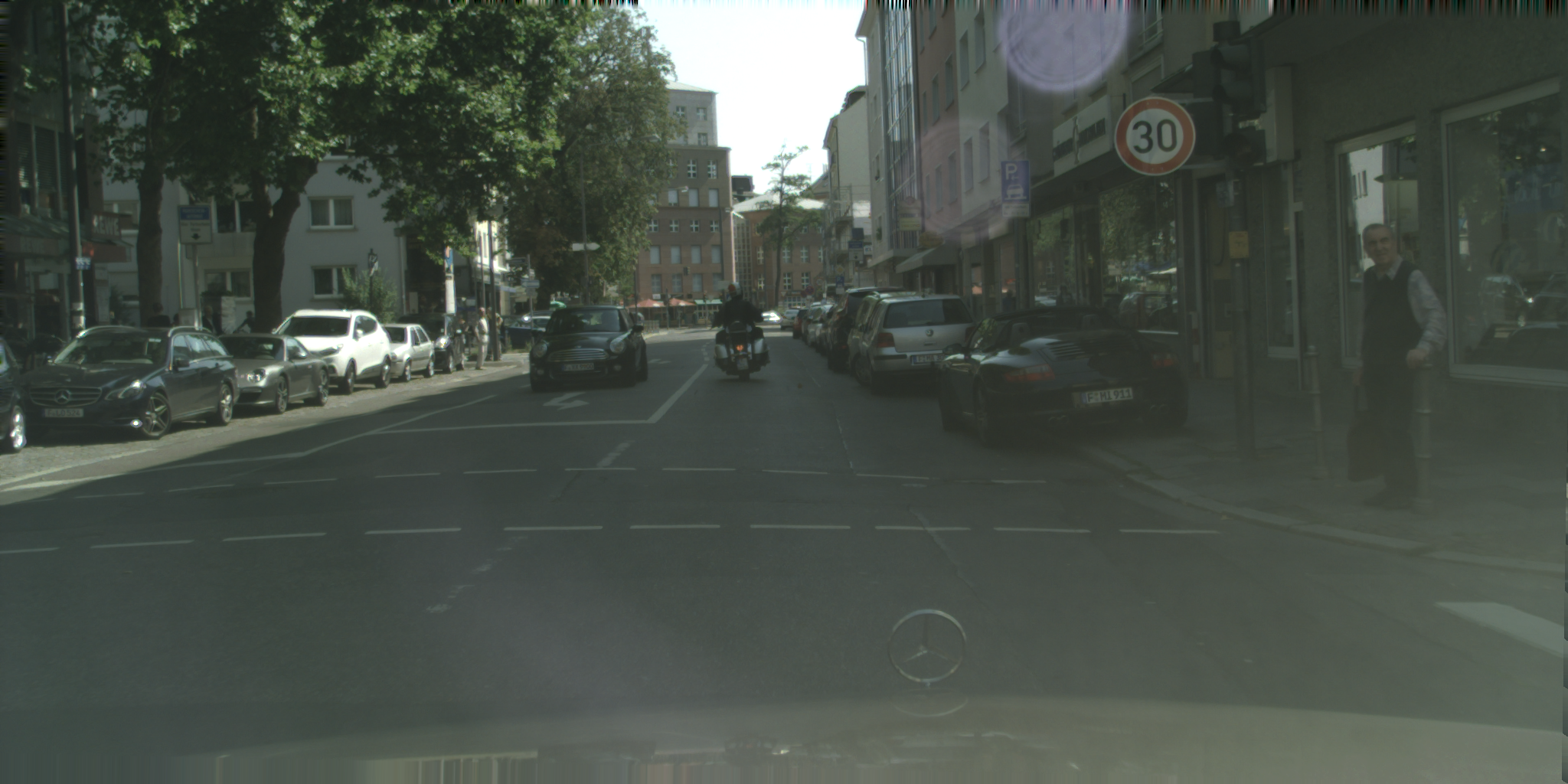}
        \caption{Attacked image}
    \end{subfigure}
    \hfill
    \begin{subfigure}[b]{\subfigwidth}
        \includegraphics[width=\textwidth]{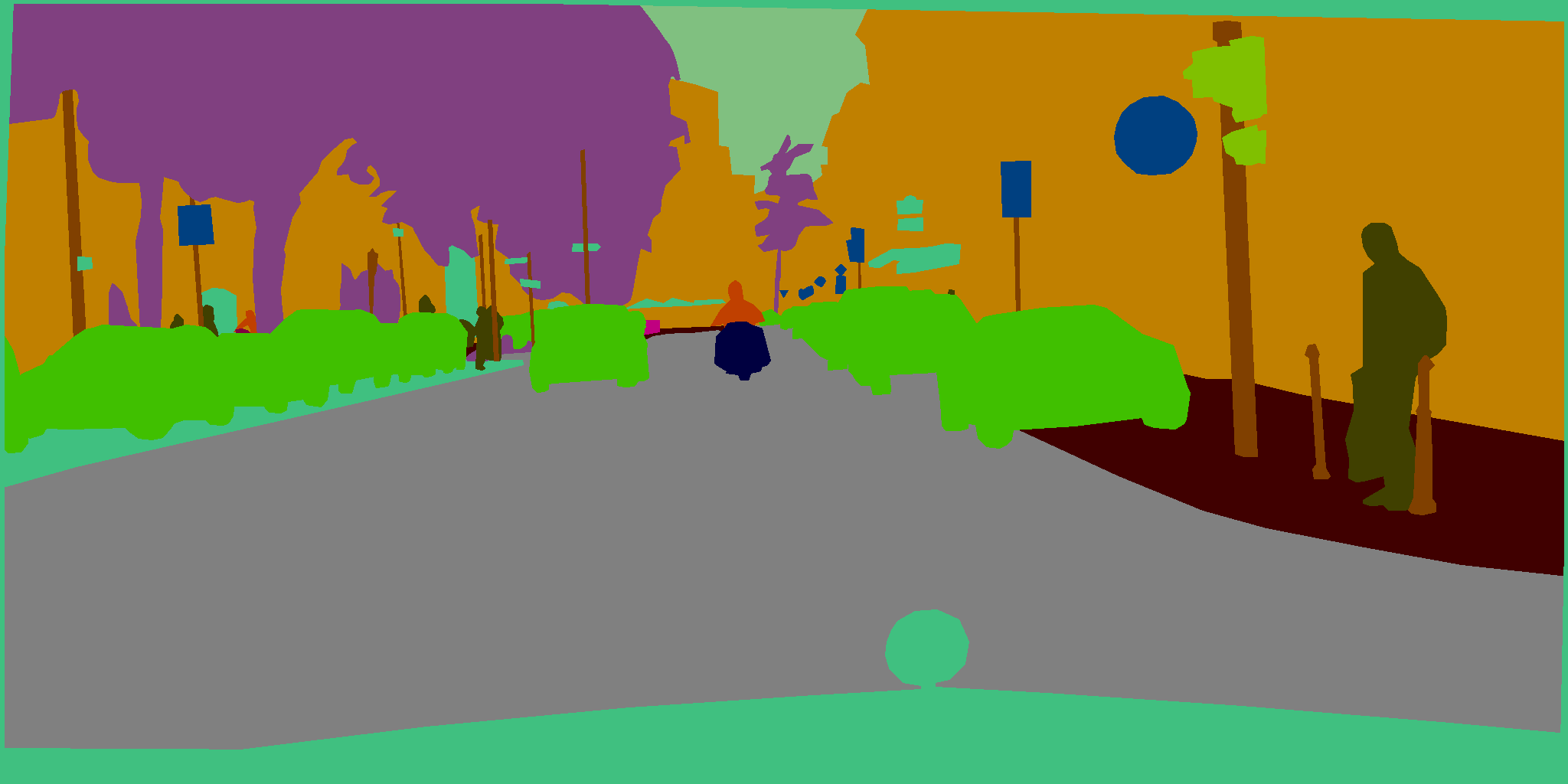}
        \caption{Ground truth seg.}
    \end{subfigure}
    \hfill
    \begin{subfigure}[b]{\subfigwidth}
        \includegraphics[width=\textwidth]{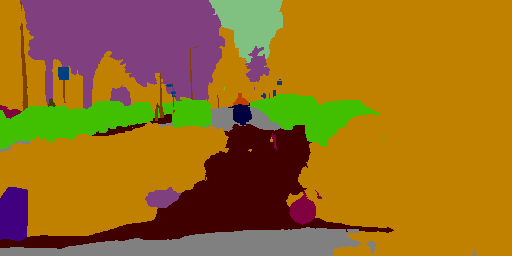}
        \caption{Attacked segmentation}
    \end{subfigure}
    \hfill
    \begin{subfigure}[b]{\subfigwidth}
        \includegraphics[width=\textwidth]{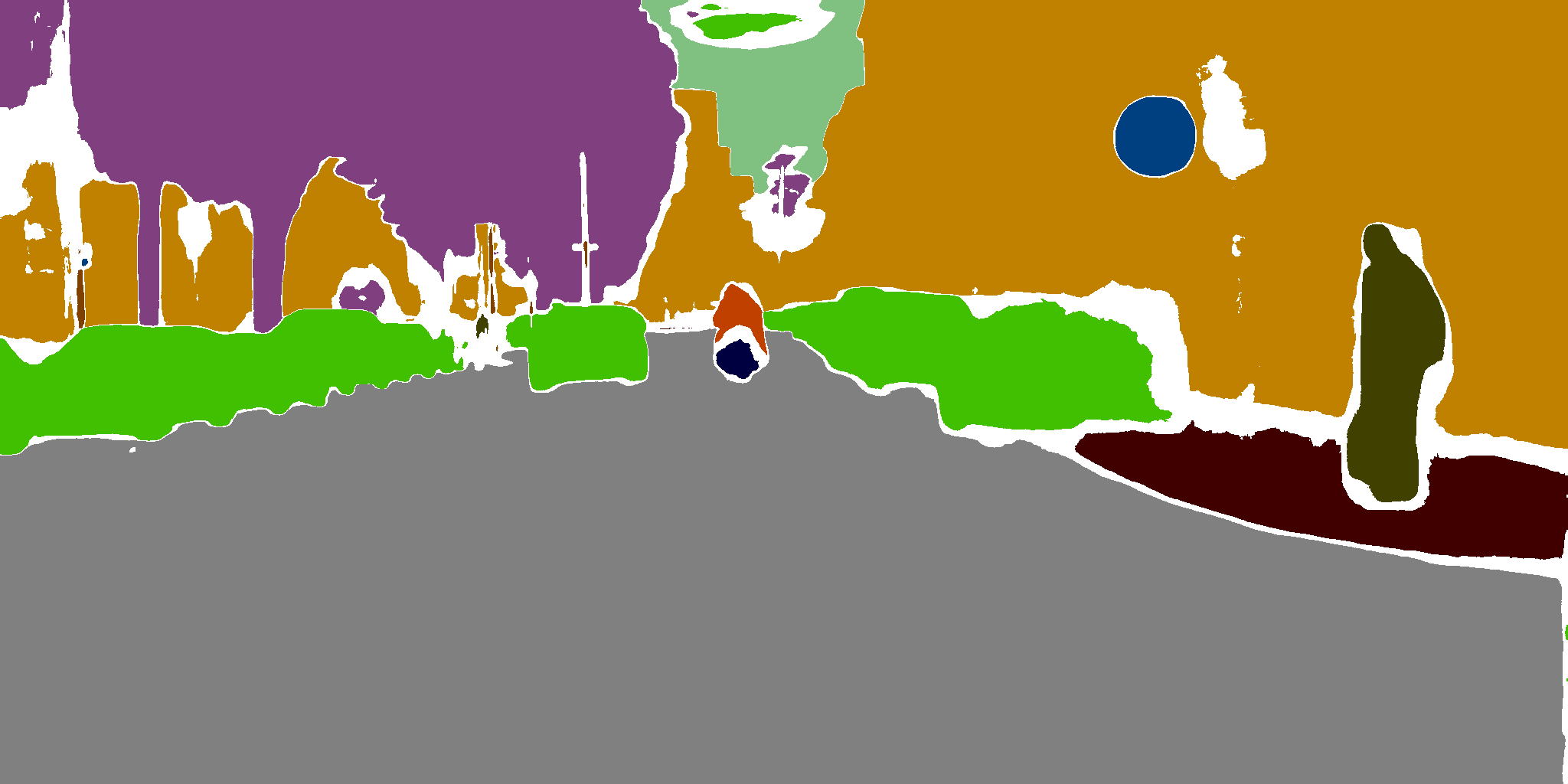}
        \caption{Certified segmentation}
    \end{subfigure}
    \caption{Another hand-picked example like \cref{fig:intro}.}
    \label{fig:add_vis_hand}
\end{figure*}

\section{Details on Attacks \& \cref{fig:intro}}
\label{sec:app-details-attacks}

To produce the visualization in \cref{fig:intro} we used a custom PGD attack
described below
and produced an $\ell_{2}$ adversarial example within a range of $0.16$.
We performed certification with
$n=100,\alpha=0.001,\sigma=0.25$ and $\tau=0.75$, which certifies robustness at an radius of
$R = \sigma \Phi^{-1}(\tau) = 0.1686$.
We perform this certification on the clean input and thus show robustness to the attack.
As the non-robust model for \cref{fig:intro:attacked} we used a
pretrained HrNet from the same repository as outlined in
\cref{sec:details-semantic}.

We use $k = 100$ steps for the attack and step size $s = \frac{10 * R }{k}$.
While scaling can be simply incorporated into the attack, we use scale $1.0$ for
both the attacked an the certified classifier.

We used an Nvidia Titan RTX to perform these attacks as the memory requirement exceeds that of a single Nvidia GeForce RTX 2080 Ti.

Here, in \cref{fig:add_vis_hand,fig:add_vis} we provide further visualizations like \cref{fig:intro}.
The visualization in \cref{fig:add_vis_hand} is hand-picked (like \cref{fig:intro}), while the ones in
\cref{fig:add_vis} are chosen randomly from the evaluated images.

\subsection{PGD for Segmentation}
\label{sec:pgd-segmentation}

Following the work of \citet{MadryMSTV18}, \citet{ArnabMT18} and
\citet{XieWZZXY17} we use a slightly generalized form of the untargeted $\ell_{2}$
projected gradient decent (PGD) attack. This version is the same as untargeted
PGD \citep{MadryMSTV18} in the classification setting, but we adapt the loss to
be the average of all pixel losses as in \citet{XieWZZXY17}. Formally, for an
input $\vx \in \mathcal{X}^{N} = [0, 1]^{N\times3}$ with ground truth segmentation
$\vy \in \mathbb{Y}^{N}$ and model $f$ a radius $R$ and stepsize $s \in \mathbb{R}$ we produce an adversarial example $\vx'_{k}$ after $k$ steps.

\begin{align*}
&\vx'_{0} = \vx + \epsilon, \qquad \epsilon \sim [0,1] \text{ with } \|\epsilon\|_{2} \leq R\\
&\vx'_{i+1} = c_{R, \vx}(\vx'_{i} + s \nabla_{\vx'_{i}} \mathcal{L}(\vx, \vy))\\
\end{align*}

with clamping function
\begin{align*}
  &c_{R} \colon \mathbb{R}^{N \times 3} \to [0, 1]^{N \times 3}\\
  &c_{R, \vx}(\vx') = \left[\vx + R \frac{\vx' - \vx}{\|\vx' - \vx\|_{2}}\right],
\end{align*}
where $[\cdot]$ denotes component-wise clamping to $[0, 1]$, and loss
\begin{equation}
 \mathcal{L}(\vx, \vy) = \frac{1}{N} \sum_{i = 1}^{N} \mathcal{H}(f_{i}(\vx), \vy_{i}),
\end{equation}
where $\mathcal{H}$ denotes the cross entropy function.

\begin{figure*}[ht]
    \newcommand{\subfigwidth}{0.4\columnwidth}
    \center
    \begin{subfigure}[b]{\subfigwidth}
        \includegraphics[width=\textwidth]{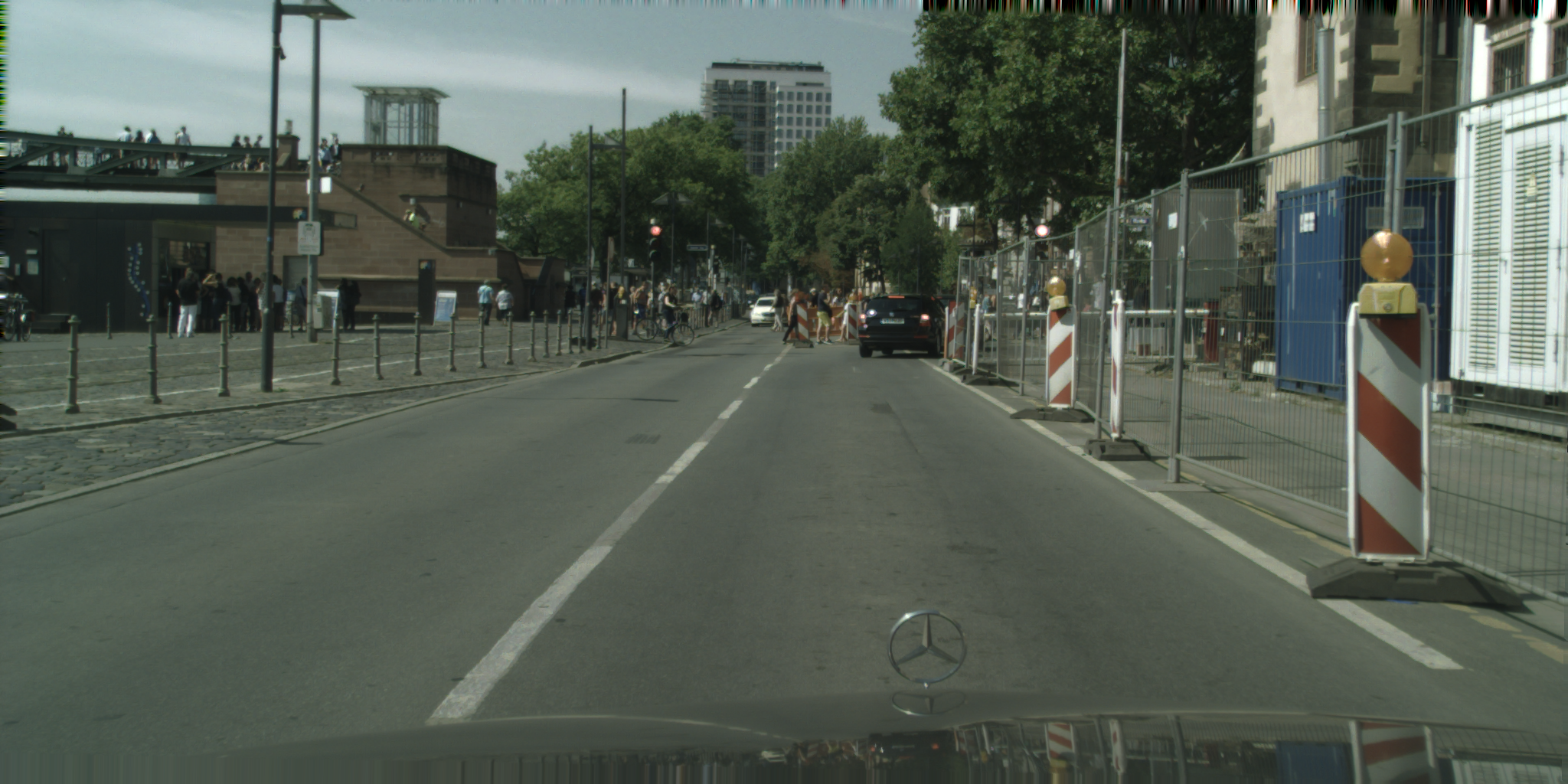}
    \end{subfigure}
    \hfill
    \begin{subfigure}[b]{\subfigwidth}
        \includegraphics[width=\textwidth]{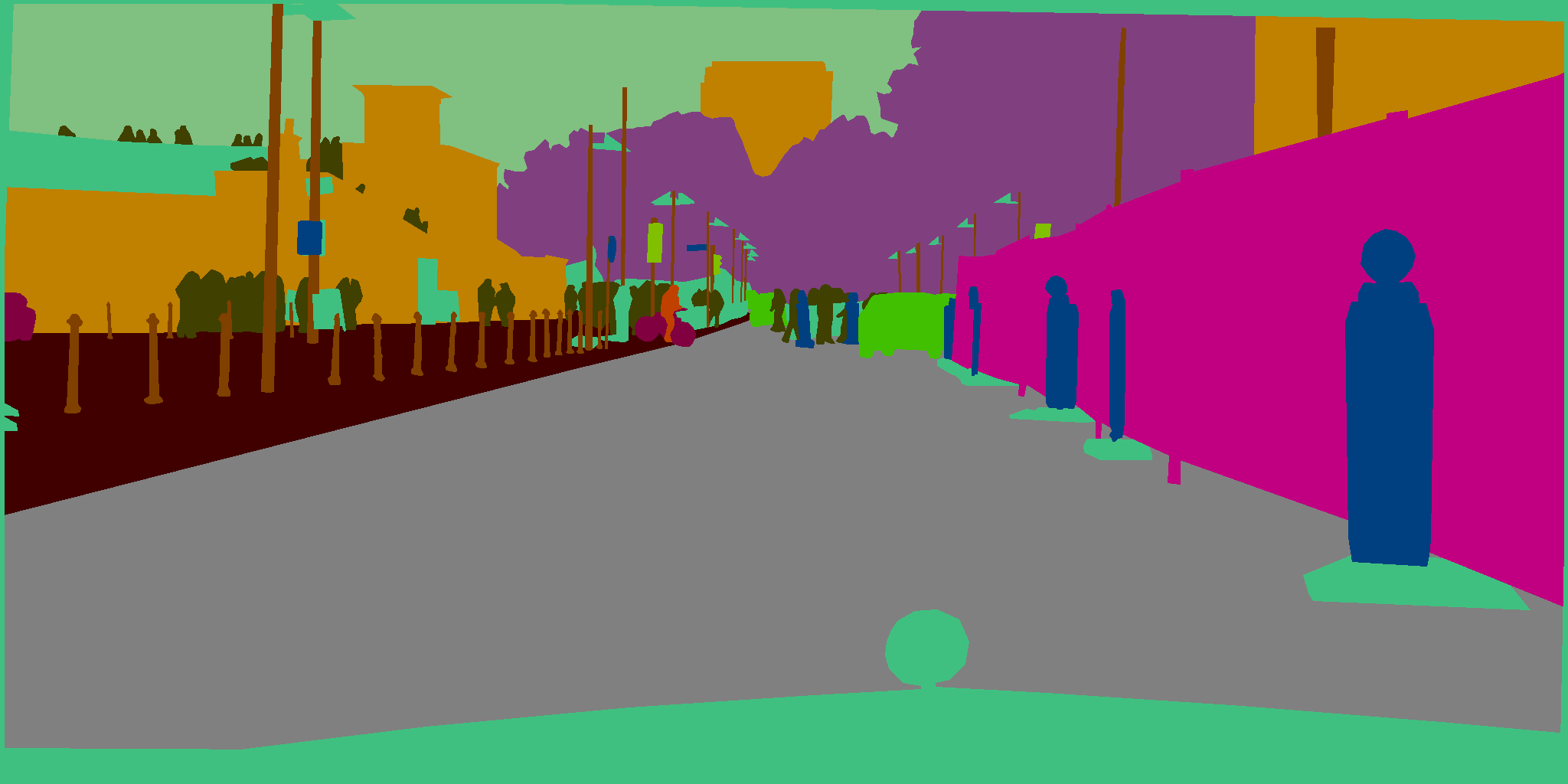}
    \end{subfigure}
    \hfill
    \begin{subfigure}[b]{\subfigwidth}
        \includegraphics[width=\textwidth]{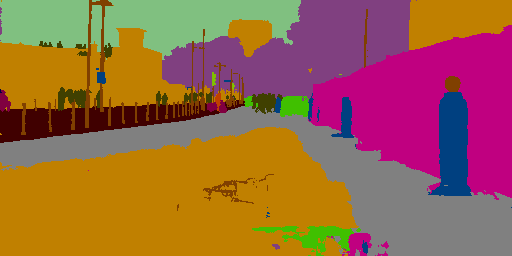}
    \end{subfigure}
    \hfill
    \begin{subfigure}[b]{\subfigwidth}
        \includegraphics[width=\textwidth]{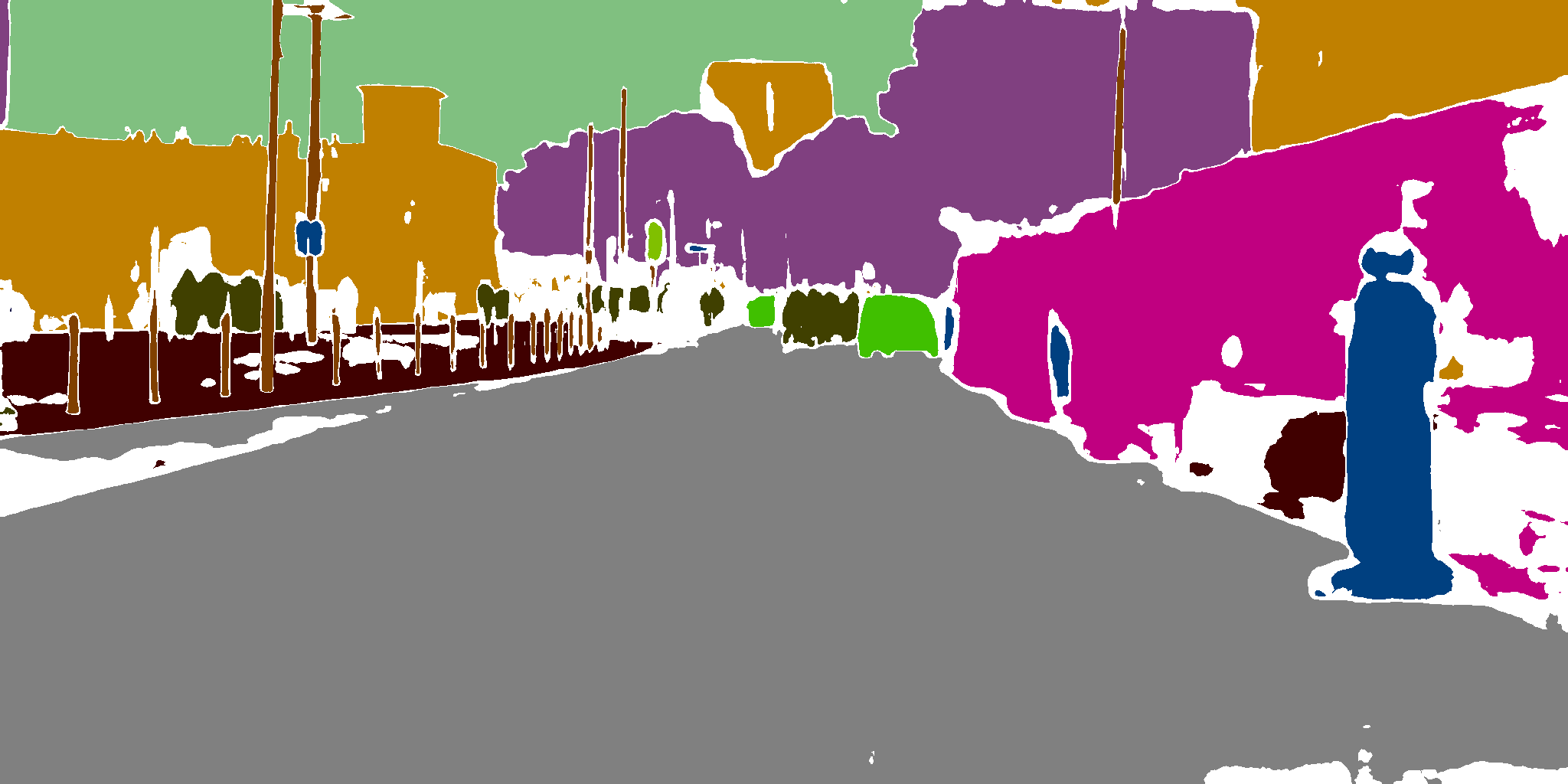}
    \end{subfigure}\\[2em]

    \begin{subfigure}[b]{\subfigwidth}
        \includegraphics[width=\textwidth]{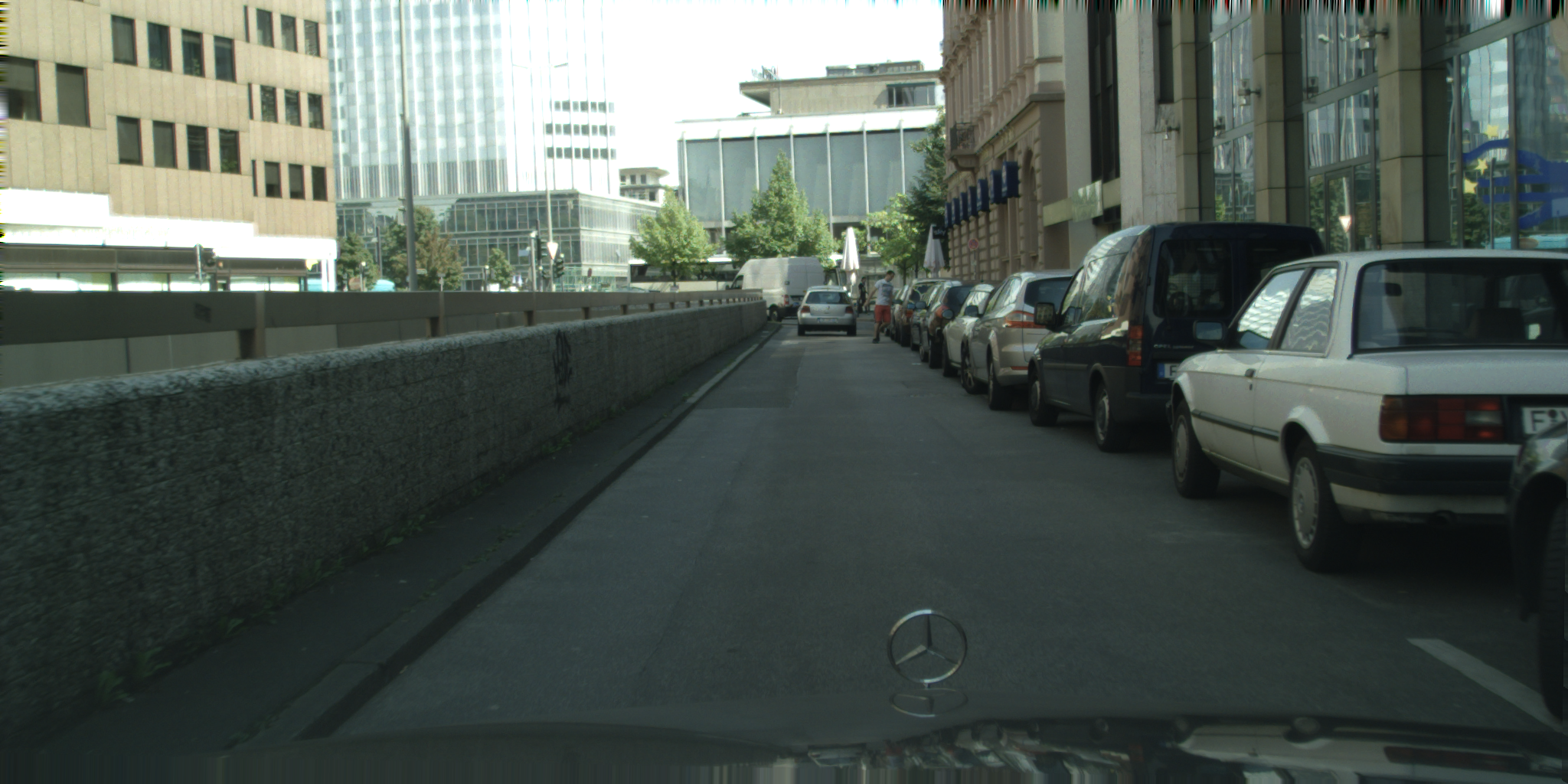}
    \end{subfigure}
    \hfill
    \begin{subfigure}[b]{\subfigwidth}
        \includegraphics[width=\textwidth]{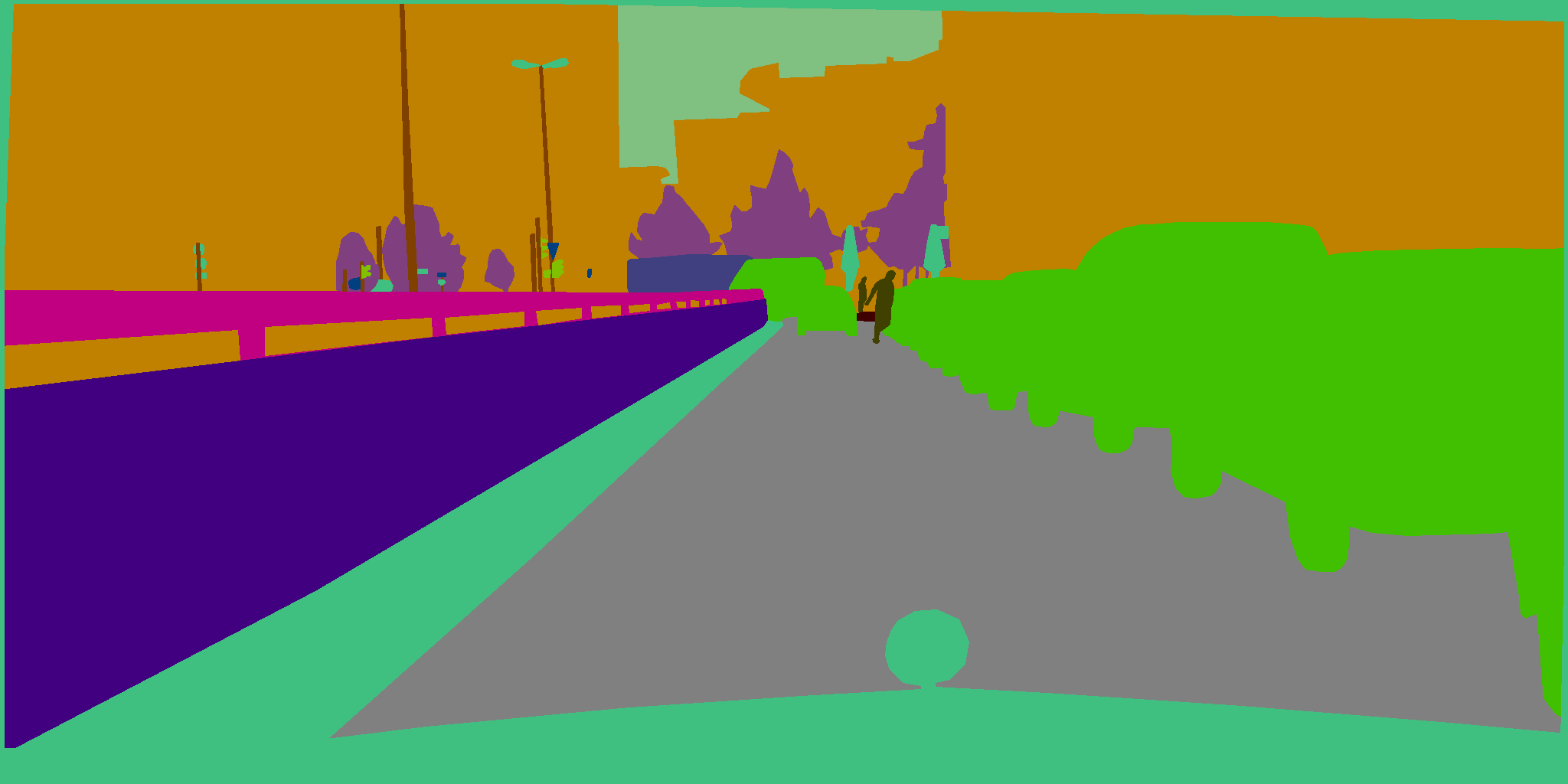}
    \end{subfigure}
    \hfill
    \begin{subfigure}[b]{\subfigwidth}
        \includegraphics[width=\textwidth]{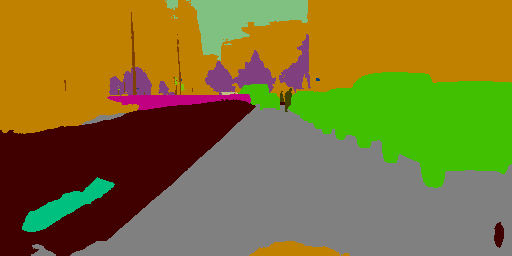}
    \end{subfigure}
    \hfill
    \begin{subfigure}[b]{\subfigwidth}
        \includegraphics[width=\textwidth]{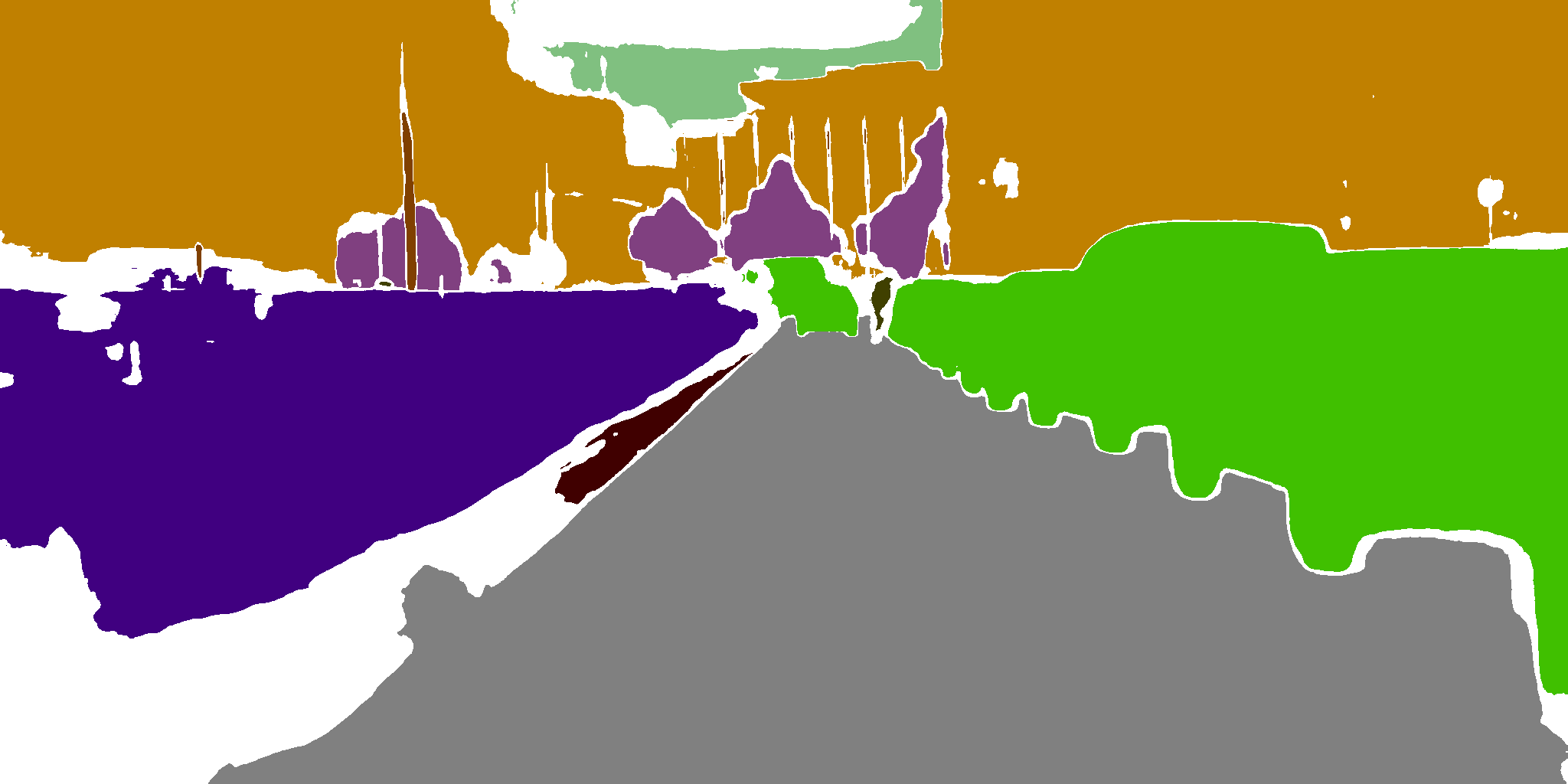}
    \end{subfigure}\\[2em]

    \begin{subfigure}[b]{\subfigwidth}
        \includegraphics[width=\textwidth]{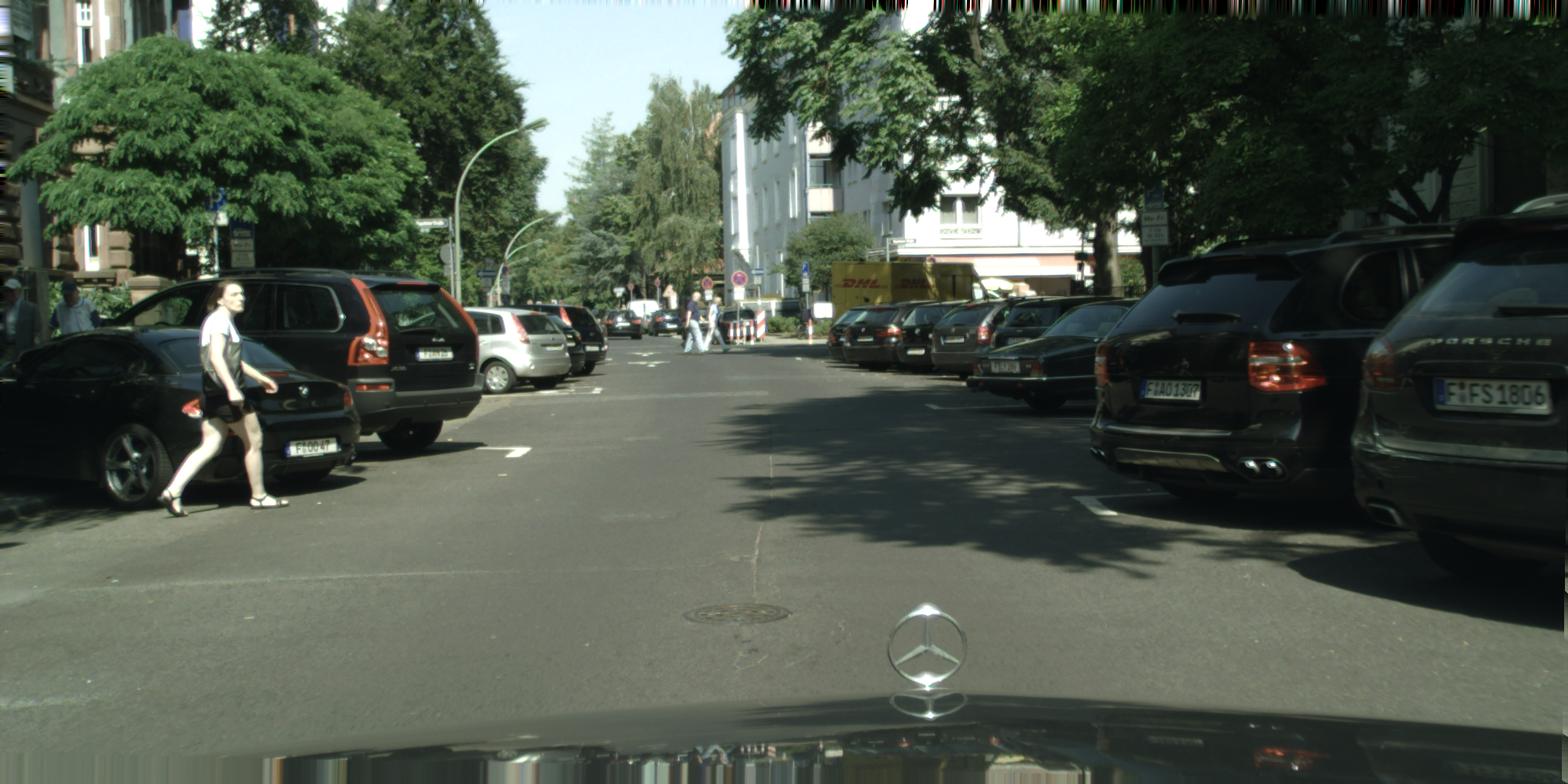}
    \end{subfigure}
    \hfill
    \begin{subfigure}[b]{\subfigwidth}
        \includegraphics[width=\textwidth]{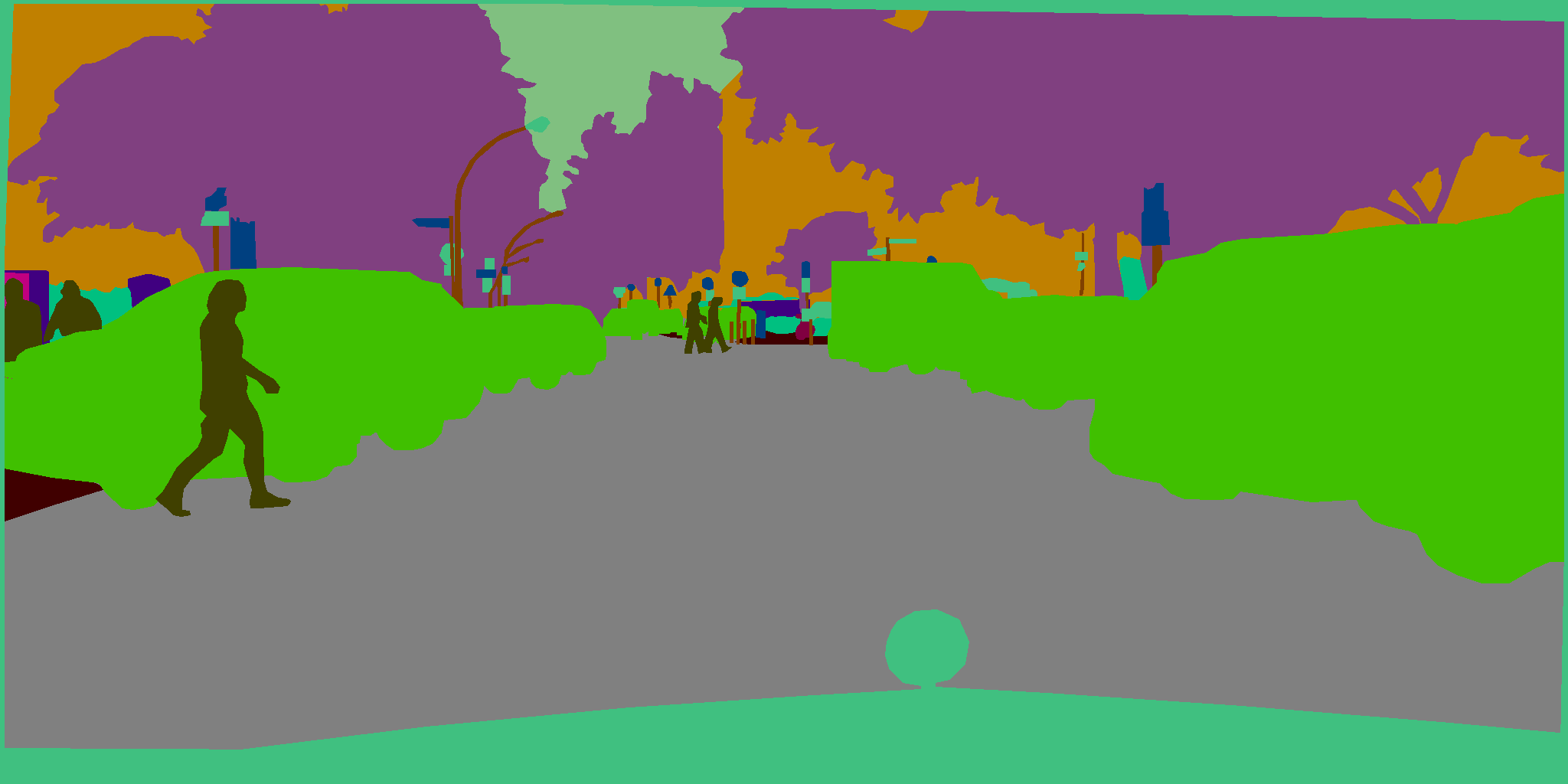}
    \end{subfigure}
    \hfill
    \begin{subfigure}[b]{\subfigwidth}
        \includegraphics[width=\textwidth]{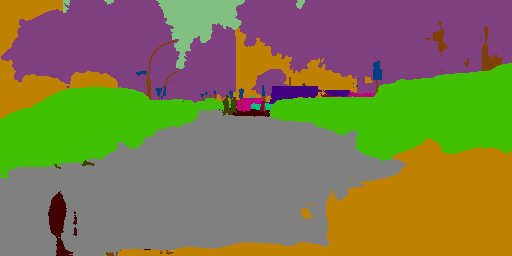}
    \end{subfigure}
    \hfill
    \begin{subfigure}[b]{\subfigwidth}
        \includegraphics[width=\textwidth]{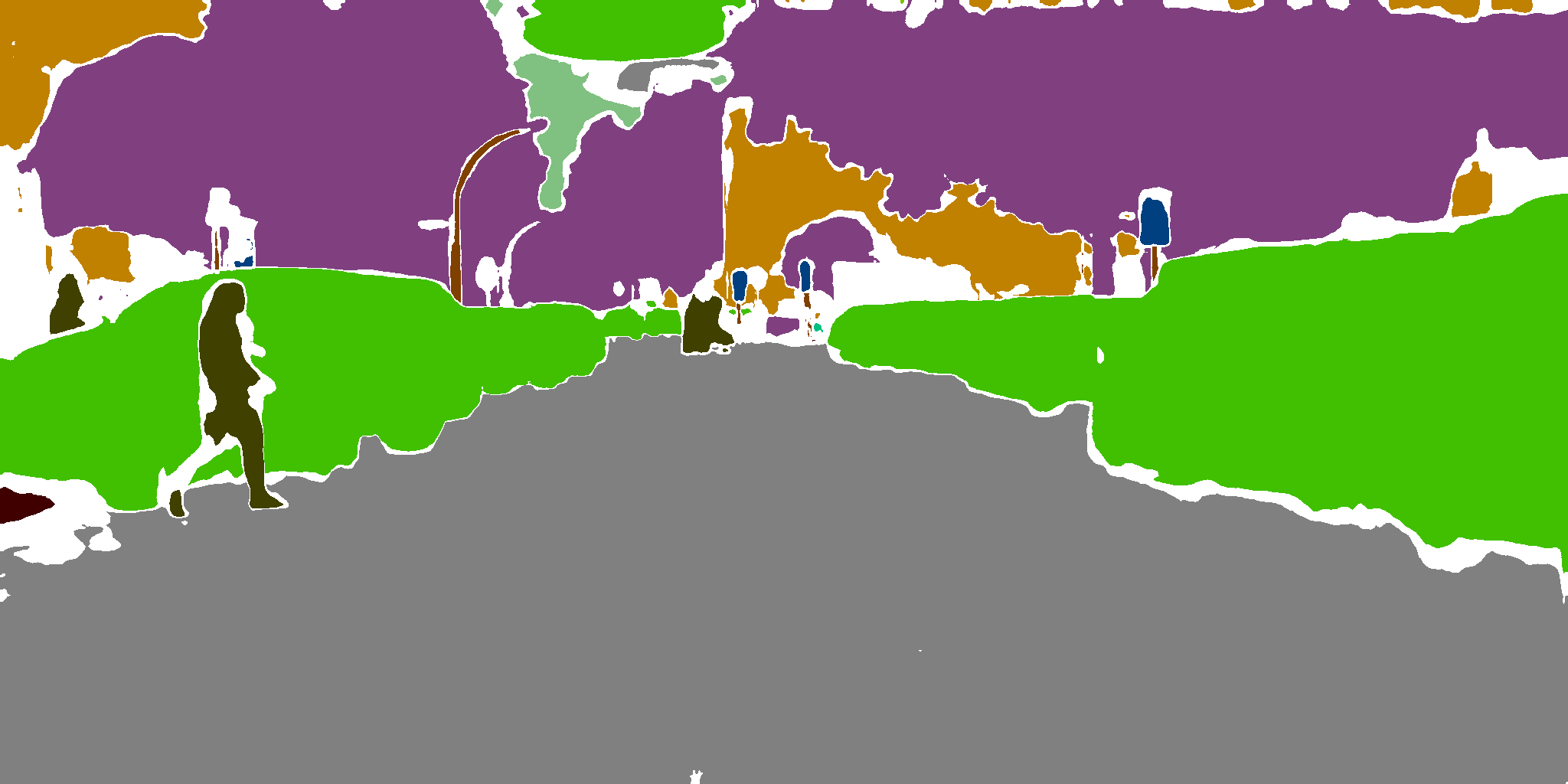}
    \end{subfigure}\\[2em]

    \begin{subfigure}[b]{\subfigwidth}
        \includegraphics[width=\textwidth]{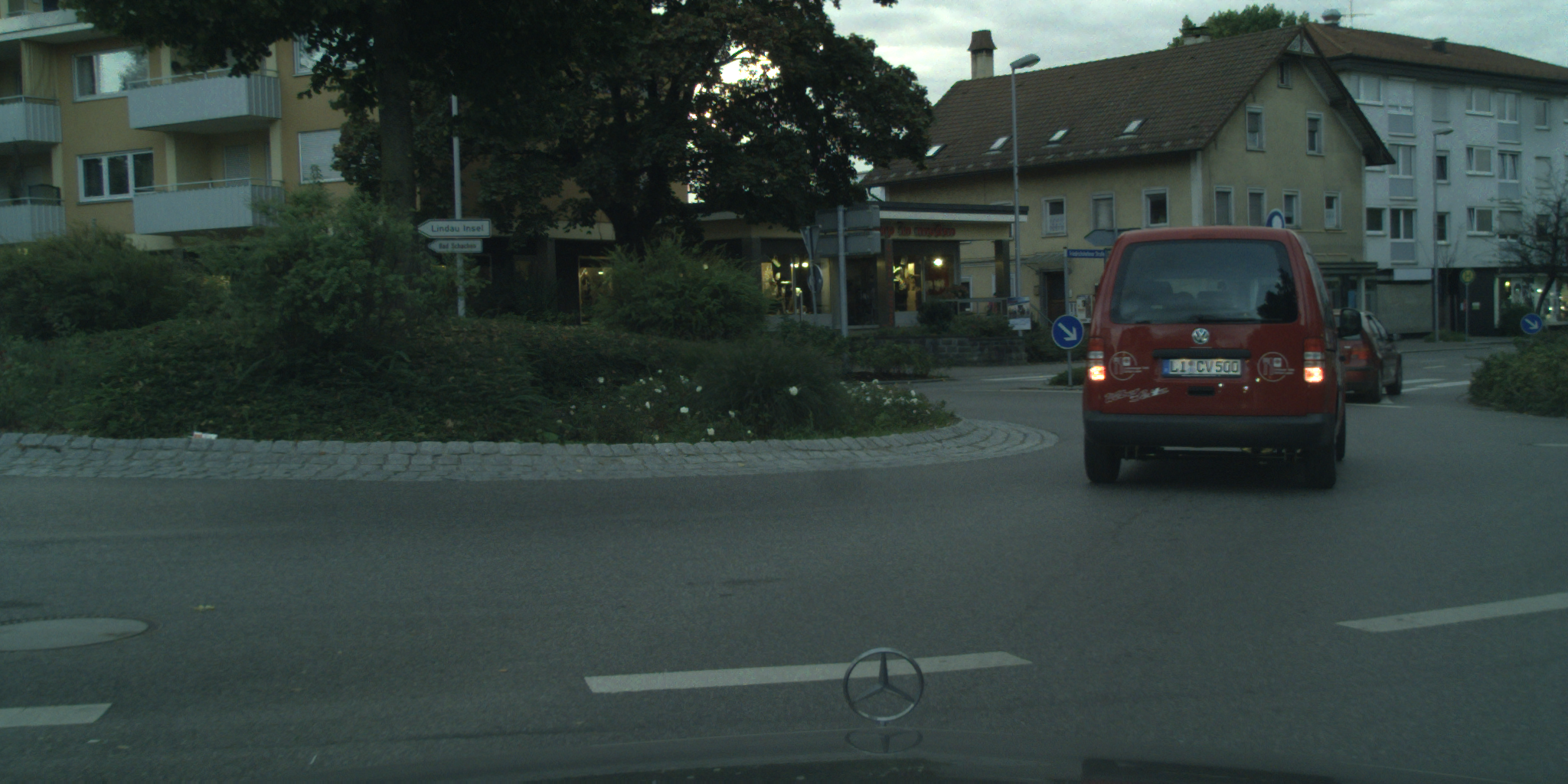}
    \end{subfigure}
    \hfill
    \begin{subfigure}[b]{\subfigwidth}
        \includegraphics[width=\textwidth]{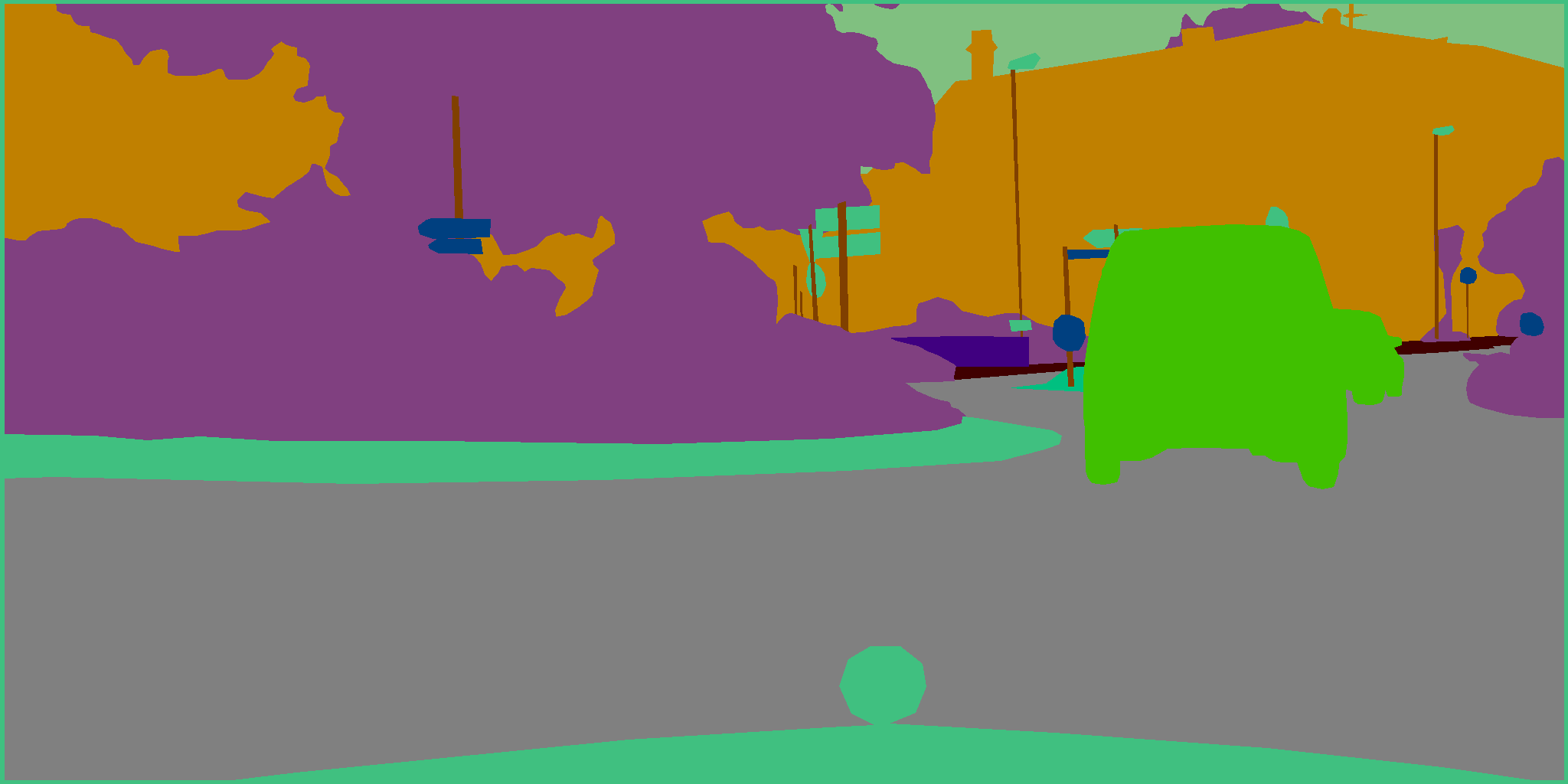}
    \end{subfigure}
    \hfill
    \begin{subfigure}[b]{\subfigwidth}
        \includegraphics[width=\textwidth]{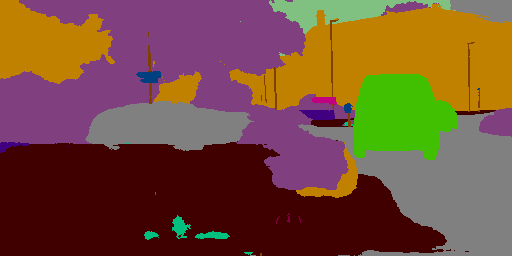}
    \end{subfigure}
    \hfill
    \begin{subfigure}[b]{\subfigwidth}
        \includegraphics[width=\textwidth]{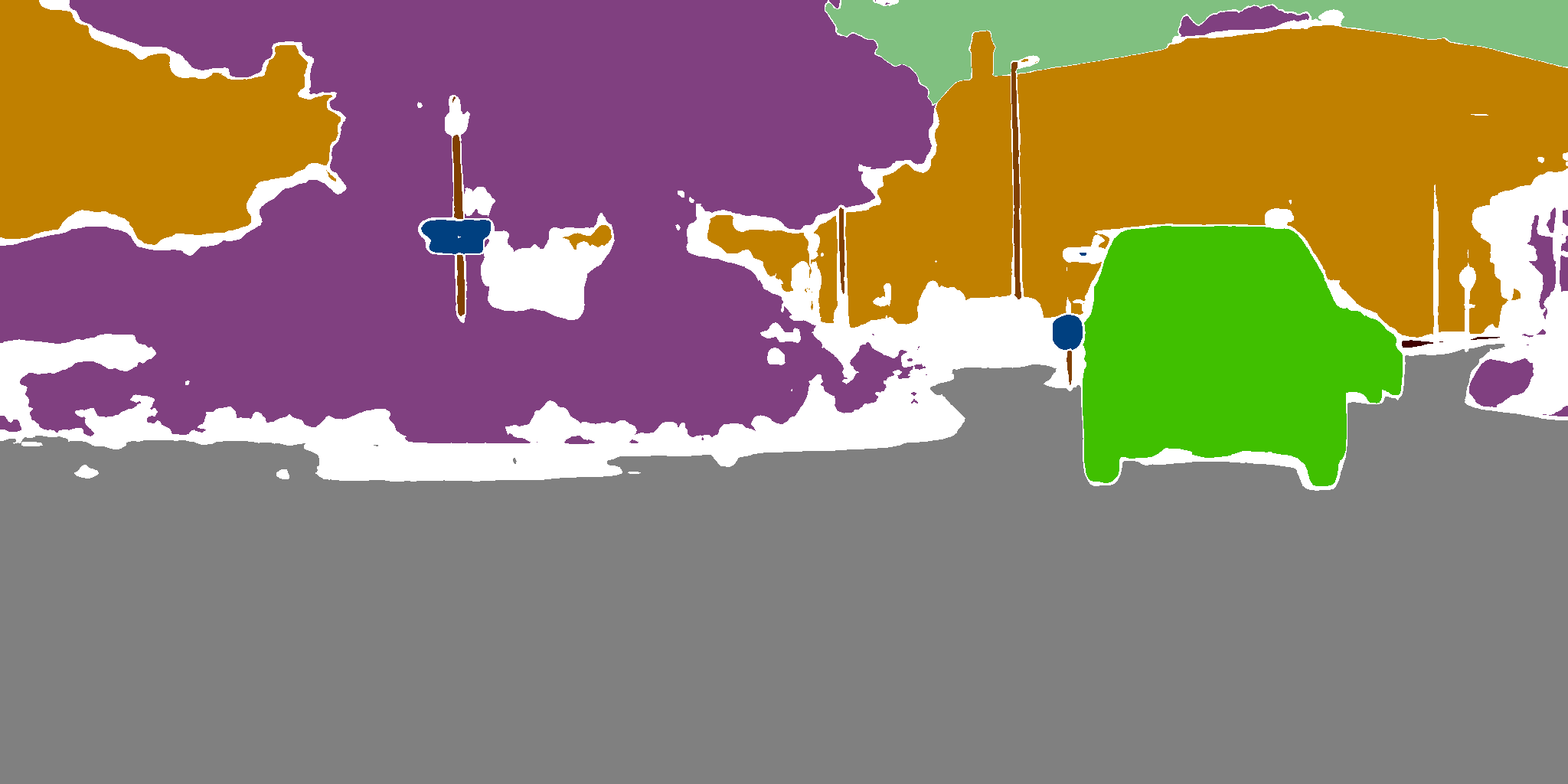}
    \end{subfigure}\\[2em]

    \begin{subfigure}[b]{\subfigwidth}
        \includegraphics[width=\textwidth]{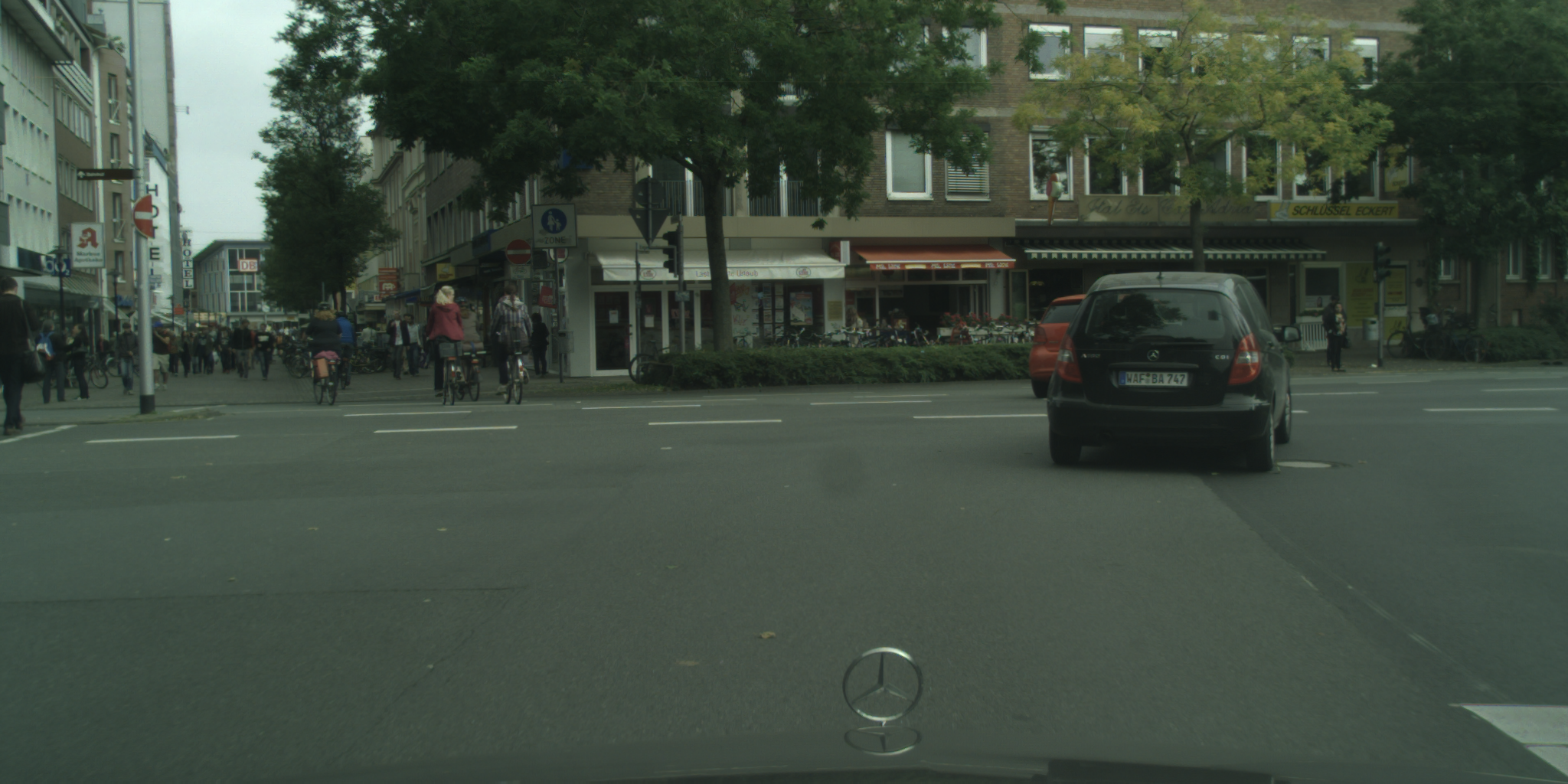}
        \caption{Attacked image}
    \end{subfigure}
    \hfill
    \begin{subfigure}[b]{\subfigwidth}
        \includegraphics[width=\textwidth]{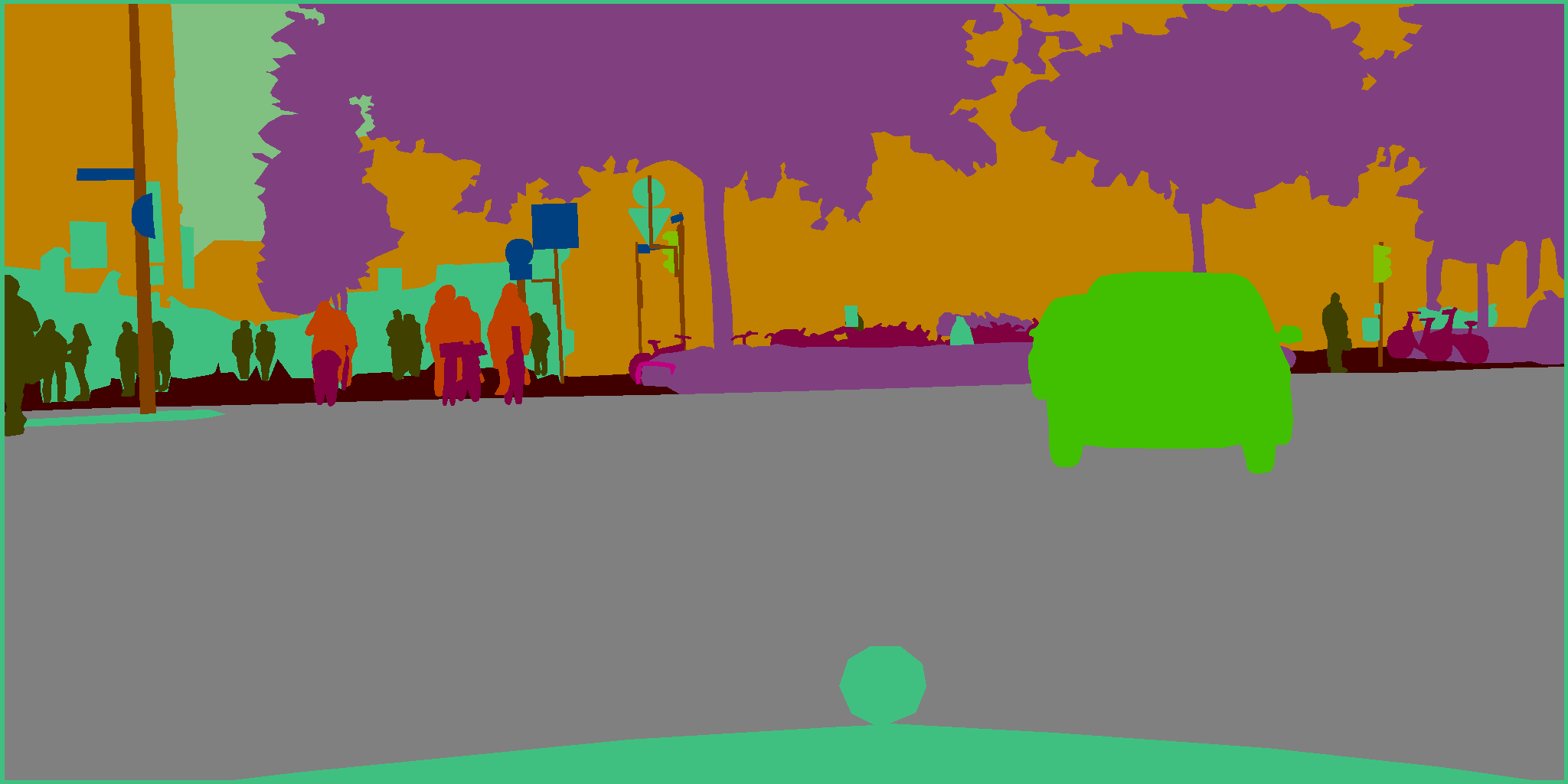}
        \caption{Ground truth seg.}
    \end{subfigure}
    \hfill
    \begin{subfigure}[b]{\subfigwidth}
        \includegraphics[width=\textwidth]{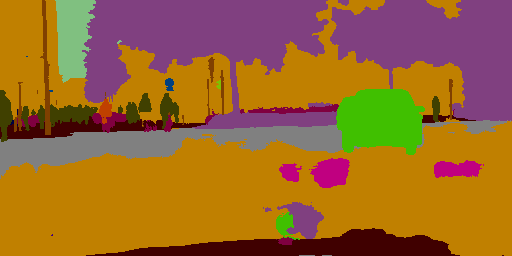}
        \caption{Attacked segmentation}
    \end{subfigure}
    \hfill
    \begin{subfigure}[b]{\subfigwidth}
        \includegraphics[width=\textwidth]{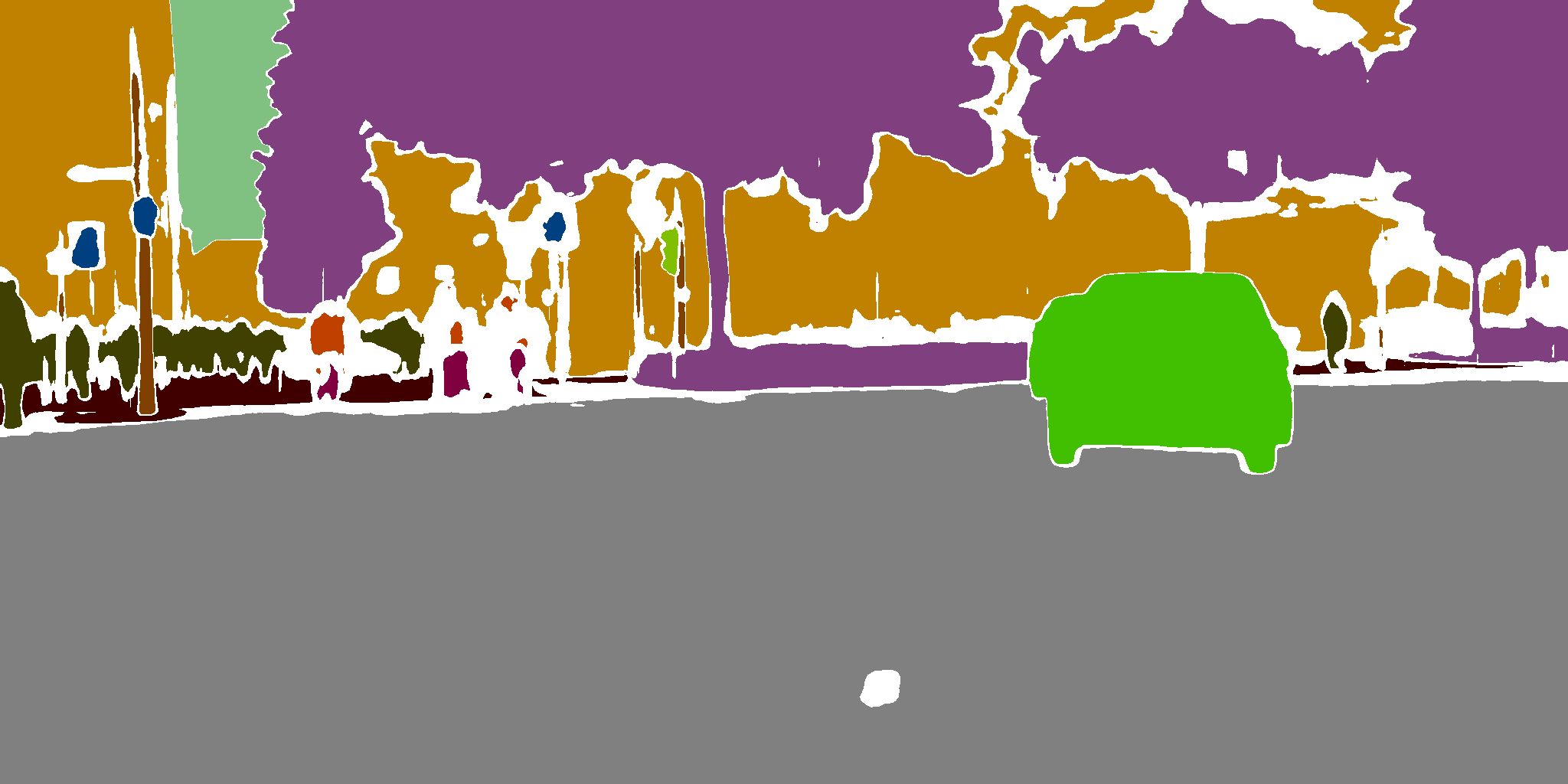}
        \caption{Certified segmentation}
    \end{subfigure}\\

    \caption{Randomly chosen examples like \cref{fig:intro}.}
    \label{fig:add_vis}
\end{figure*}